\documentclass[11pt]{article} 

\usepackage{times}
\usepackage{makeidx}
\usepackage{amsfonts}
\usepackage{fullpage}
\usepackage{algorithm}
\usepackage{algorithmic}
\usepackage{dsfont}
\usepackage{hyperref}
\usepackage{microtype}
\hypersetup{colorlinks,
            linkcolor=blue,
            citecolor=blue,
            urlcolor=magenta,
            linktocpage,
            plainpages=false}
\usepackage{amsmath}
\usepackage{amssymb}
\usepackage{amsthm}
\usepackage{url}
\usepackage{natbib}
\newcommand{\eps}{\epsilon}
\newcommand{\on}{\{-1,1\}}

\newcommand{\pr}{\mathop{\mathbf{Pr}}}
\newcommand{\E}{\mathop{\mathbf E}}
\newcommand{\R}{\mathbb{R}}
\newcommand{\D}{\mathcal{D}}
\newcommand{\ORR}{\mathsf{OR}}
\newcommand{\ANDD}{\mathsf{AND}}

\newcommand{\U}{\mathcal{U}}
\newcommand{\N}{\mathbb{N}}
\newcommand{\A}{\mathcal{A}}

\newcommand{\I}{I}
\newcommand{\ie}{i.e.\text{ }}
\newcommand{\eg}{e.g.}
\newcommand{\size}{\mathsf{size}}
\newcommand{\T}{\mathbb{T}}
\newcommand{\CQ}{\mathbf{CQ}}
\newcommand{\C}{\mathcal{C}}

\newcommand{\poly}{\mathsf{poly}}

\newcommand{\zo}{\{0, 1\}}

\renewcommand{\S}{\mathbb{S}}
\newcommand{\F}{\mathcal{F}}
\renewcommand{\P}{\mathcal{P}}
\newtheorem{fact}{Fact}[section]
\newtheorem{definition}[fact]{Definition}

\newtheorem{theorem}[fact]{Theorem}
\newtheorem{lemma}[fact]{Lemma}
\newtheorem{corollary}[fact]{Corollary}

\newtheorem{proposition}{Proposition}[section]
\newtheorem{remark}{Remark}[section]
\newcommand{\eat}[1]{}
\newcommand{\onlyfull}[1]{#1}
\newcommand{\onlycolt}[1]{}
\newcommand{\cond}{\ |\ }
\newcommand{\alequ}[1]{\begin{align} #1 \end{align}}

\floatstyle{plain}
\newfloat{myalgo}{tbhp}{mya}

{\begin{myalgo}[#1]
\centering
\begin{minipage}{#2}
\begin{algorithm}[H]}%
{\end{algorithm}
\end{minipage}
\end{myalgo}}

 \author{Vitaly Feldman \\ vitaly@post.harvard.edu \and
Pravesh Kothari \\ kothari@cs.utexas.edu}
\title{Learning Coverage Functions and Private Release of Marginals}

\begin{document}
\date{\empty}
\maketitle

\begin{abstract}
We study the problem of approximating and learning coverage functions. A function $c: 2^{[n]} \rightarrow \R^{+}$ is a coverage function, if there exists a universe $U$ with non-negative weights $w(u)$ for each $u \in U$ and subsets $A_1, A_2, \ldots, A_n$ of $U$ such that $c(S) = \sum_{u \in \cup_{i \in S} A_i} w(u)$. Alternatively, coverage functions can be described as non-negative linear combinations of monotone disjunctions. They are a natural subclass of submodular functions and arise in a number of applications.

We give an algorithm that for any $\gamma,\delta>0$, given random and uniform examples of an unknown coverage function $c$, finds a function $h$ that approximates $c$ within factor $1+\gamma$ on all but $\delta$-fraction of the points in time $\poly(n,1/\gamma,1/\delta)$. This is the first fully-polynomial algorithm for learning an interesting class of functions in the demanding PMAC model of  \citet{BalcanHarvey:12full}.
Our algorithms are based on several new structural properties of coverage functions.
Using the results in \citep{FeldmanKothari:14symm}, we also show that coverage functions are learnable agnostically with excess $\ell_1$-error $\eps$ over all product and symmetric distributions in time $n^{\log(1/\eps)}$.
In contrast, we show that, without assumptions on the distribution, learning coverage functions is at least as hard as learning polynomial-size disjoint DNF formulas, a class of functions for which the best known algorithm runs in time $2^{\tilde{O}(n^{1/3})}$ \citep{KS04}.

As an application of our learning results, we give simple differentially-private algorithms for releasing monotone conjunction counting queries with low {\em average} error. 
In particular, for any $k \leq n$, we obtain private release of $k$-way marginals with average error $\bar{\alpha}$ in time $n^{O(\log(1/\bar{\alpha}))}$.
\end{abstract}




\newcommand{\Cv}{\mathcal{CV}}
\section{Introduction}
We consider learning and approximation of the class of \emph{coverage} functions over the Boolean hypercube $\on^n$. A function $c: 2^{[n]} \rightarrow \R^{+}$ is a coverage function if there exists a family of sets $A_1, A_2, \ldots, A_n$ on a universe $U$ equipped with a weight function $w:U \rightarrow \R^{+}$ such that for any $S \subseteq [n]$, $c(S) = w( \cup_{i \in S} A_i),$ where $w(T) = \sum_{u \in T} w(u)$ for any $T \subseteq U$. We view these functions over $\on^n$ by associating each subset $S \subseteq [n]$ with vector $x^S \in \on^n$ such that $x^S_i = -1$ iff $i \in S$. We define the size (denoted by $\size(c)$) of a coverage function $c$ as the size of a smallest-size universe $U$ that can be used to define $c$. As is well-known, coverage functions also have an equivalent and natural representation as non-negative linear combinations of monotone disjunctions with the size being the number of disjunctions in the combination. 

Coverage functions form a relatively simple but important subclass of the broad class of submodular functions. Submodular functions have been studied in a number of contexts and play an important role in combinatorial optimization \citep{Lov83,GW95,FFI01,Edm70,Fra97,Fei98} with several applications to machine learning \citep{GKS05,KGGK06,KG11,IyerBilmes:13} and in algorithmic game theory, where they are used to model valuation functions \citep{LLN06,DS06,Vondrak08}. Coverage functions themselves figure in several applications such as facility location \citep{CFN77}, private data release of conjunctions \citep{GHRU11} and algorithmic game theory where they are used to model the utilities of agents in welfare maximization and design of combinatorial auctions \citep{DV11}.

In this paper, we investigate the learnability of coverage functions from random examples. The study of learnability from random examples of the larger classes of functions such as submodular and fractionally-subadditive functions has been initiated by \citet{BalcanHarvey:12full} who were motivated by applications in algorithmic game theory. They introduced the PMAC model of learning in which, given random and independent examples of an unknown function, the learner is required to output a hypothesis that is multiplicatively close (which is the standard notion of approximation in the optimization setting) to the unknown target on at least $1-\delta$ fraction of the points. This setting is also considered in \citep{BalcanCIW:12,BDFKNR12}. Learning of submodular functions with less demanding (and more common in machine learning) additive guarantees was first considered by \citet{GHRU11}, who were motivated by problems in private data release. In this setting the goal of the learner is equivalent to producing a hypothesis that $\epsilon$-approximates the target function in $\ell_1$ or $\ell_2$ distance. That is for functions $f,g$, $\E_{x \sim \D} [|f(x)-g(x)|]$ or $\sqrt{\E_{x \sim \D}[(f(x)-g(x))^2]}$ where $\D$ is the underlying distribution on the domain (with the uniform distribution being the most common). The same notion of error and restriction to the uniform distribution are also used in several subsequent works on learning of submodular functions \citep{CKKL12,RY13,FeldmanKV:13,FeldmanVondrak:13}. We consider both these models in the present work. For a more detailed survey of submodular function learning the reader is referred to \citep{BalcanHarvey:12full}.

\subsection{Our Results}
\subsubsection{Distribution-independent learning} Our main results are for the uniform, product and symmetric distribution learning of coverage functions. However it is useful to first understand the complexity of learning these functions without any distributional assumptions (for a formal definition and details of the models of learning see Sec. \ref{sec:prelims}). We prove (see Sec. ~\ref{app:distind}) that distribution-independent learning of coverage functions is at least as hard as PAC learning the class of polynomial-size disjoint DNF formulas over arbitrary distributions (that is DNF formulas, where each point satisfies at most 1 term). Polynomial-size disjoint DNF formulas is an expressive class of Boolean functions that includes the class of polynomial-size decision trees, for example. Moreover, there is no known algorithm for learning polynomial-size disjoint DNFs that runs faster than the algorithm for learning general DNF formulas, the best known algorithm for which runs in time $2^{\tilde{O}(n^{1/3})}$ \citep{KS04}. Let $\Cv$ denote the class of coverage functions over $\on^n$ with range in $[0,1]$.
\begin{theorem}
\label{th:dnf-reduction-intro}
Let $\A$ be an algorithm that learns all coverage functions in $\Cv$ of size at most $s$ with $\ell_1$-error $\epsilon$ in time $T(n, s, \frac{1}{\epsilon})$. Then, there exists an algorithm $\A'$ that PAC learns the class of $s$-term disjoint-DNF in time $T(2n, s, \frac{2s}{\epsilon})$.
\end{theorem}
This reduction gives a computational impediment to fully-polynomial PAC (and consequently PMAC) learning of coverage functions of polynomial size or any class that includes coverage functions. Previously, hardness results for learning various classes of submodular and fractionally-subadditive functions were information-theoretic \citep{BalcanHarvey:12full,BDFKNR12,BalcanCIW:12} or required encodings of cryptographic primitives in the function \citep{BalcanHarvey:12full}. 
\onlyfull{
On the positive side, in Sec.~\ref{sec:thresholdreduction} we show that learning (both distribution-specific and distribution-independent) of coverage functions of size $s$ is at most as hard as learning the class of linear thresholds of $s$ monotone Boolean disjunctions (which for example include monotone CNF with $s$ clauses). A special case of this simple reduction appears in \citep{HRS12}.
}

\subsubsection{PAC and PMAC learning over the uniform distribution} Learning of submodular functions becomes substantially easier when the distribution is restricted to be uniform (denoted by $\U$). For example, all submodular functions are learnable with $\ell_1$-error of $\eps$ in time $2^{O(1/\eps^4)} \cdot \poly(n)$ \citep{FeldmanKV:13} whereas there is a constant $\alpha$ and a distribution $\D$ such that no polynomial-time algorithm can achieve $\ell_1$-error of $\alpha$ when learning submodular functions relative to $\D$ \citep{BalcanHarvey:12full}. At the same time achieving fully-polynomial time is often hard even under this strong assumption on the distribution. For example, polynomial-size disjoint DNF or monotone DNF/CNF are not known to be learnable efficiently in this setting and the best algorithms run in $n^{O(\log{(n/\eps)})}$ time. But, as we show below, when restricted to the uniform distribution, coverage functions are easier than disjoint DNF and are PAC learnable efficiently. Further, they are learnable in fully-polynomial time even with the stronger multiplicative approximation guarantees of the PMAC learning model \citep{BalcanHarvey:12full}. We first state the PAC learning result which is easier to prove and serves as a step toward the PMAC algorithm.
\begin{theorem}
There exists an algorithm which, given $\eps > 0$ and access to random uniform examples of any coverage function $c \in \Cv$, with probability at least $2/3$, outputs a hypothesis $h$ such that $\E_\U[ |h(x) -c(x)|] \leq \epsilon$. The algorithm runs in  $\tilde{O}(n/\epsilon^4 + 1/\eps^8)$ time and uses $\log{n} \cdot \tilde{O}(1/\epsilon^4)$ examples. \label{thm:pacintro}
\end{theorem}
We note that for general submodular functions exponential dependence on $1/\eps$ is necessary information-theoretically \citep{FeldmanKV:13}. To obtain an algorithm with multiplicative guarantees we show that for every monotone submodular (and not just coverage) function multiplicative approximation can be easily reduced to additive approximation. The reduction decomposes $\on^n$ into $O(\log(1/\delta))$ subcubes where the target function is relatively large with high probability, specifically the value of $f$ on each subcube is $\Omega(1/\log(1/\delta))$ times the maximum value of $f$ on the subcube. The reduction is based on concentration results for submodular functions \citep{BLM00,Von10,BalcanHarvey:12full} and the fact that for any non-negative monotone submodular function $f$, $\E_\U[f] \geq \|f\|_\infty/2$ \citep{Feige:06}.
This reduction together with Thm.~\ref{thm:pacintro} yields our PMAC learning algorithm for coverage functions.
\begin{theorem} \label{thm:PMACintro}
There exists an algorithm which, given $\gamma,\delta > 0$ and access to random uniform examples of any coverage function $c$, with probability at least $2/3$, outputs a hypothesis $h$ such that $\pr_\U[ h(x) \leq c(x) \leq(1+\gamma) h(x)] \geq 1-\delta$. The algorithm runs in  $\tilde{O}(\frac{n}{\gamma^4 \delta^4} + \frac{1}{\gamma^8\delta^8})$  time and uses $\log{n} \cdot \tilde{O} (\frac{1}{\gamma^4 \delta^4})$ examples.
\end{theorem}
This is the first fully-polynomial (that is polynomial in $n$, $1/\eps$ and $1/\delta$) algorithm for PMAC learning a natural subclass of submodular functions even when the distribution is restricted to be uniform. As a point of comparison, the sketching result of  \citet{BDFKNR12} shows that for every coverage function $c$ and $\gamma >0$, there exists a coverage function of size $\poly(n,1/\gamma)$ size that approximates $c$ within factor $1+\gamma$  everywhere.  Unfortunately, it is unknown how to compute this strong approximation even in subexponential time and even with value queries\footnote{A value query on a point in a domain returns the value of the target function at the point. For Boolean functions it is usually referred to as a membership query.} and the distribution is restricted to be uniform. . Our result shows that if one relaxes the approximation to be over $1-\delta$ fraction of points then in time polynomial in $n, 1/\gamma$ and $1/\delta$ one can find a $(1+\gamma)$-approximating function using random examples alone.

The key property that we identify and exploit in designing the PAC algorithm is that the Fourier coefficients of coverage functions have a form of (anti-)monotonicity property.
\begin{lemma}
\label{lem:monotonicity}
For a coverage function $c: \on^n \rightarrow [0,1]$ and non-empty $T \subseteq V$, $|\hat{c}(T)| \geq |\hat{c}(V)|$.
\end{lemma}
This lemma allows us to find all significant Fourier coefficients of a coverage function efficiently using a search procedure analogous to that in the Kushilevitz-Mansour algorithm \citep{KushilevitzMansour:93} (but without the need for value queries). 
An additional useful property we prove is that any coverage function can be approximated by a function of few variables (referred to as {\em junta}).
\begin{theorem}
\label{thm:junta-coverageintro}
For any coverage function $c:\on^n \rightarrow [0,1]$ and $\eps >0$, there exists a coverage function $c'$, that depends only on $O(1/\epsilon^2)$ variables and satisfies $\E_\U[|c(x)-c'(x)|] \leq \epsilon$.
\end{theorem}
By identifying the variables of an approximating junta we make the learning algorithm computationally more efficient and achieve logarithmic dependence of the number of random examples on $n$. This, in particular, implies {\em attribute efficiency} \citep{BlumLangley:97} of our algorithm. Our bound on junta size is tight since coverage functions include monotone linear functions which require a $\Omega(1/\eps^2)$-junta for $\eps$-approximation (\eg~\citep{FeldmanVondrak:13}). This clearly distinguishes coverage functions from disjunctions themselves which can always be approximated using a function of just $O(\log(1/\eps))$ variables. We note that in a subsequent work \citet{FeldmanVondrak:13} showed approximation by $O(\log(1/\eps)/\eps^2)$-juntas for all submodular functions using a more involved approach. They also show that this approximation leads to a $2^{\tilde{O}(1/(\delta\gamma)^2)} \cdot \poly(n)$ PMAC learning algorithm for all submodular functions.

Exploiting the representation of coverage functions as non-negative linear combinations of monotone disjunctions, we show that we can actually get a PAC learning algorithm that outputs a hypothesis that is guaranteed to be a coverage function. That is, the algorithm is \emph{proper}. The running time of this algorithm is polynomial in $n$ and, in addition, depends polynomially on the size of the target coverage function. 
\onlyfull{
\begin{theorem}
\label{thm:properpacintro}
 There exists an algorithm, that for any $\epsilon > 0$, given random and uniform examples of any $c \in \Cv$, with probability at least $2/3$, outputs a \emph{coverage} function $h$ such that $\E_{\U}[ |h(x)-c(x)|] \leq \epsilon$. The algorithm runs in time $\tilde{O}(n) \cdot \poly(s/\epsilon)$ and uses $\log{(n)} \cdot \poly(s/\epsilon)$ random examples, where $s = \min\{\size(c), (1/\eps)^{\log{(1/\epsilon)}}\}$.
\end{theorem}}

\subsubsection{Agnostic Learning on Product and Symmetric Distributions} We then consider learning of coverage functions over general product and symmetric distributions (that is those whose PDF is symmetric with respect to the $n$ variables). These are natural generalizations of the uniform distribution studied in a number of prior works. In our case the motivation comes from the application to differentially-private release of (monotone) $k$-conjunction counting queries referred to as \emph{$k$-way marginals} in this context. Releasing $k$-way marginals with average error corresponds to learning of coverage functions over the uniform distribution on points of Hamming weight $k$ which is a symmetric distribution (we describe the applications in more detail in the next subsection).

As usual with Fourier transform-based techniques, on general product distributions the running time of our PAC learning algorithm becomes polynomial in $(1/p)^{O(\log(1/\eps))}$, where $p$ is the smallest bias of a variable in the distribution. It also relies heavily on the independence of variables and therefore does not apply to general symmetric distributions. Therefore, we use a different approach to the problem which learns coverage functions by learning disjunctions in the agnostic learning model \citep{Haussler:92,KearnsSS:94}. This approach is based on a simple and known observation that if disjunctions can be approximated in $\ell_1$ distance by linear combinations of some basis functions then so are coverage functions. As a result, the learning algorithm for coverage functions also has agnostic guarantees relative to $\ell_1$-error.

\begin{theorem}\label{th:introsym}
There exists an algorithm which for any product or symmetric distribution $\D$ on $\on^n$, given $\eps > 0$ and access examples of a function $f:\on^n \rightarrow [0,1]$ on points sampled from $\D$, with probability at least $2/3$, outputs a hypothesis $h$ such that $\E_\D[ |h(x) - f(x)|] \leq \min_{c \in \Cv}\{ \E_\D[ |c(x) - f(x)|] \} + \eps$.
The algorithm runs in $n^{O(\log{(1/\epsilon)})}$ time.
\end{theorem}
For product distributions, this algorithm relies on the fact that disjunctions can be $\ell_1$-approximated within $\eps$ by degree $O(\log(1/\eps))$ polynomials \citep{BOW08}.
\onlycolt{
We give a different and simpler proof for this fact. One advantage of our proof is that a very similar argument works for symmetric distributions, although the basis functions are no longer monomials. In fact, we show that there exists a distribution $\D$ such that any polynomial that approximates the disjunction $x_1 \vee x_2 \vee \cdots  \vee x_n$ within an $\ell_1$ error of at most $1/3$ on $\D$ must be of degree $\Omega(\sqrt{n})$. To the best of our knowledge such approximation for disjunctions and separation from approximation by polynomials were not known before.
}
\onlyfull{
A simpler proof for this approximation appears in \citep{FeldmanKothari:14symm} where it is also shown that the same result holds for all symmetric distributions.
}

\onlyfull{
For the special case of product distributions that have their one dimensional marginal expectations bounded away from $0$ and $1$ by some constants, we show that we can in fact make our agnostic learning algorithm \emph{proper}, that is, the hypothesis returned by our algorithm is a coverage function.
In particular,
we give a proper agnostic learning algorithm for coverage functions over the uniform distribution running in time $n^{O(\log{(1/\epsilon)})}$. It is not hard to show that this algorithm is essentially the best possible assuming hardness of learning sparse parities with noise.
}
\subsubsection{Applications to Differentially Private Data Release}
We now briefly overview the problem of differentially private data release and state our results. Formal definitions and details of our applications to privacy appear in Sec.~\ref{sec:privacy} and a more detailed background discussion can for example be found in \citep{TUV12}. The objective of a private data release algorithm is to release answers to all {\em counting queries} from a given class $\C$ with low error while protecting the privacy of participants in the data set. Specifically, we are given a data set $D$ which is a subset of a fixed domain $X$ (in our case $X = \on^n$). Given a query class $\C$ of Boolean functions on $\on^n$, the objective is to output a data structure $H$ that allows answering {\em counting queries} from $\C$ on $D$ with low error. A counting query for $c \in \C$ gives the fraction of elements in $D$ on which $c$ equals to $1$. The algorithm producing $H$ should be \emph{differentially private} \citep{DMNS06}. The efficiency of a private release algorithm for releasing a class of queries $\C$ with error $\alpha$ on a data set $D \subseteq X$ is measured by its running time (in the size of the data set, the dimension and the error parameter) and the minimum data set size required for achieving certain error. Informally speaking, a release algorithm is differentially private if adding an element of $X$ to (or removing an element of $X$ from) $D$ does not affect the probability that any specific $H$ will be output by the algorithm significantly. A natural and often useful way of private data release for a data set $D$ is to output another data set $\hat{D} \subseteq X$ (in a differentially private way) such that answers to counting queries based on $\hat{D}$ approximate answers based on $D$. Such release is referred to as data {\em sanitization} and the data set is referred to as a {\em synthetic data set}.

Releasing Boolean conjunction counting queries is likely the single best motivated and most well-studied problem in private data analysis \citep{BCDKMT07,UV10,CKKL12,HRS12,TUV12,BUV13,CTUW14,DNT13}. It is a part of the official statistics in the form of reported data in the US Census, Bureau of Labor statistics and the Internal Revenue Service.

Despite the relative simplicity of this class of functions, the best known algorithm for releasing all $k$-way marginals with a constant worst-case error runs in polynomial time for data sets of size at least $n^{\Omega(\sqrt{k})}$ \citep{TUV12}. Starting with the work of \citet{GHRU11}, researchers have also considered the private release problem with low \emph{average} error with respect to some distribution, most commonly uniform, on the class of queries \citep{CKKL12,HRS12,DNT13}. However, in most applications only relatively short marginals are of interest and therefore the average error relative to the uniform distribution can be completely uninformative in this case. As can be easily seen (\eg~\citep{GHRU11}), the function mapping a monotone conjunction $c$ to a counting query for $c$ on a data set $D$ can be written in terms of a convex combination of monotone disjunctions corresponding to points in $D$ which is a coverage function. In this translation the distribution on conjunctions becomes a distribution over points on which the coverage function is defined and the $\ell_1$ error in approximating the coverage function becomes the average error of the data release. Therefore using standard techniques, we adapt our learning algorithms to
this problem. Thm.~\ref{th:introsym} gives the following algorithm for release of $k$-way marginals. 
\begin{theorem}\label{th:introsymprivacy}
Let $\C_k$ be the class of all monotone conjunctions of length $k \in [n]$. For every $\epsilon > 0$, there is an $\epsilon$-differentially private algorithm which for any data set $D \subseteq \on^n$ of size $n^{\Omega(\log{(1/\bar{\alpha})})} \cdot \log{(1/\delta)} /\epsilon$, with probability at least $1-\delta$ outputs a data structure $H$ that answers counting queries for $\C_k$ with respect to the uniform distribution on $\C_k$ with an average error of at most $\bar{\alpha}$. The algorithm runs in time $n^{O(\log{(1/\bar{\alpha})})} \cdot \log{(1/\delta)}/\epsilon$ and the size of $H$ is $n^{O(\log{(1/\bar{\alpha})})}$.
\end{theorem}
Note that there is no dependence on $k$ in the bounds and it applies to any symmetric distribution. 
Without assumptions on the distribution, \citet{DNT13} give an algorithm that releases $k$-way marginals with average error $\bar{\alpha}$ given a data set of size at least $\tilde{\Omega}(n^{\lceil k/2 \rceil/2} \cdot 1/\bar{\alpha}^2)$ and runs in polynomial time in this size. (They also give a method to obtain the stronger worst-case error guarantees by using \emph{private boosting}.) 

We then adapt our PAC learning algorithms for coverage functions to give two algorithms for privately releasing monotone conjunction counting queries over the uniform distribution. Our first algorithm uses Thm.~\ref{thm:pacintro} to obtain a differentially private algorithm for releasing monotone conjunction counting queries in time polynomial in $n$ (the data set dimension) and $1/\bar{\alpha}$.
\begin{theorem}
\label{thm:nosynth-intro}
Let $\C$ be the class of all monotone conjunctions. For every $\eps,\delta > 0$, there exists an $\epsilon$-differentially private algorithm which for any data set $D \subseteq \on^n$ of size $\tilde{\Omega}(n \log(1/\delta)/(\eps\bar{\alpha}^6))$, with probability at least $1-\delta$, outputs a data structure $H$ that answers counting queries for $\C$ with respect to the uniform distribution with an average error of at most $\bar{\alpha}$. The algorithm runs in time $\tilde{O}(n^2 \log(1/\delta)/(\eps\bar{\alpha}^{10}))$ and the size of $H$ is $\log n \cdot \tilde{O}(1/\bar{\alpha}^4)$.
\end{theorem}
The previous best algorithm for this problem runs is time $n^{O(\log{(1/\bar{\alpha})})}$ \citep{CKKL12}. In addition, using a general framework from \citep{HRS12}, one can reduce private release of monotone conjunction counting queries to PAC learning with value queries of linear thresholds of a polynomial number of conjunctions over a certain class of ``smooth" distributions. \citet{HRS12} show how to use their framework together with Jackson's algorithm for learning majorities of parities \citep{Jackson:97} to privately release parity counting queries. Using a similar argument one can also obtain a polynomial-time algorithm for privately releasing monotone conjunction counting queries. Our algorithm is substantially simpler and more efficient than the one obtained via the reduction in \citep{HRS12}.

We can also use our proper learning algorithm to obtain a differentially private sanitization for releasing marginals in time polynomial in $n$ and quasi-polynomial in $1/\bar{\alpha}$.
\onlycolt{ See Section \ref{sec:pdr} for more details.
}

\onlyfull{
\begin{theorem}
\label{thm:synth-intro}
Let $\C$ be the class of all monotone conjunctions. For every $\eps,\delta > 0$, there exists an $\epsilon$-differentially private algorithm which for any data set $D \subseteq \on^n$ of size $n \cdot \bar{\alpha}^{-\Omega(\log{(1/\bar{\alpha})})} \cdot \log(1/\delta)/\epsilon$, with probability at least $1-\delta$, releases a synthetic data set $\hat{D}$ that can answer counting queries for $\C$ with respect to the uniform distribution with average error of at most $\bar{\alpha}$. The algorithm runs in time $n^2 \cdot \bar{\alpha}^{-O(\log{(1/\bar{\alpha})})} \cdot \log(1/\delta)/\epsilon$.
\end{theorem}
Note that our algorithm for privately releasing monotone conjunction queries with low-average error via a synthetic data set is polynomial time for any error $\bar{\alpha}$ that is $2^{-O(\sqrt{\log(n)})}$.
}

\subsection{Related Work}
\citet{BDFKNR12} study \emph{sketching} of coverage functions and prove that for any coverage function there exists a small (polynomial in the dimension and the inverse of the error parameter) approximate representation that multiplicatively approximates the function on all points. Their result implies an algorithm for learning coverage functions in the PMAC model \citep{BalcanHarvey:12full} that uses a polynomial number of examples but requires exponential time in the dimension $n$. \citet{CH12} study the problem of \emph{testing} coverage functions (under what they call the \emph{W-distance}) and show that the class of coverage functions of polynomial size can be reconstructed, that is, one can obtain in polynomial time, a representation of an unknown coverage function $c$ such that $\size(c)$ is bounded by some polynomial in $n$ (in general for a coverage function $c$, $\size(c)$ can be as high as $2^n$), that computes $c$ correctly at all points, using polynomially many value queries. Their reconstruction algorithm can be seen as an \emph{exact} learning algorithm with value queries for coverage functions of small size.

In a recent (and independent) work, \citet{BCY13} develop a subroutine for learning sums of monotone conjunctions that also relies on the monotonicity the Fourier coefficients (as in Lemma \ref{lem:monotonicity}). Their application is in a very different context  of learning DNF expressions from \emph{numerical pairwise} queries, which given two assignments from $\on^n$ to the variables, expects in reply, the number of terms of the target DNF satisfied by both assignments.

\onlyfull{
A general result of \citet{GHRU11} shows that releasing all counting queries from a concept class $\C$ using counting queries (when accessing the data set) requires as many counting queries as agnostically learning $\C$ using statistical queries. Using lower bounds on statistical query complexity of agnostic learning of conjunctions \citep{Feldman:12jcss} they derived a lower bound on counting query complexity for releasing all conjunction counting queries of certain length. This rules out a fully polynomial (in $k$, the data set size and the dimension $n$) algorithm to privately release short conjunction counting queries with low worst-case error.

Since our algorithms access the data set using counting queries, the lower bounds from \citep{GHRU11} apply to our setting. However the lower bound in \citep{GHRU11} is only significant when the length of conjunctions is at most logarithmic in $n$.  Building on the work of  \citet{DNRRV09}, \citet{UV10} showed that there exists a constant $\gamma$ such that there is no polynomial time algorithm  for releasing a synthetic data set that answers all conjunction counting queries with worst-case error of at most $\gamma$ under some mild cryptographic assumptions. 
}

\section{Preliminaries}
\label{sec:prelims}
We use $\on^n$ to denote the $n$-dimensional Boolean hypercube with ``false" mapped to $1$ and ``true" mapped to $-1$.
Let $[n]$ denote the set $\{1,2,\ldots,n\}$. For $S\subseteq [n]$, we denote by $\ORR_S: \on^n \rightarrow \zo$, the monotone Boolean disjunction on variables with indices in $S$, that is, for any $x \in \on^n$, $\ORR_S(x) = 0 \Leftrightarrow  \forall i \in S \ \ x_i = 1$. A monotone Boolean disjunction is a simple example of a coverage function. To see this, consider a universe of size $1$, containing a single element say $u$, the associated weight, $w(u) =1$, and the sets $A_1, A_2, \ldots, A_n$ such that $A_i$ contains $u$ if and only if $i \in S$. In the following lemma we describe a natural and folklore characterization of coverage functions as non-negative linear combination of non-empty monotone disjunctions (\eg~\citep{GHRU11}). For completeness we include the proof in App.~\ref{app:proofs}.
\begin{lemma}
A function $c: \on^n \rightarrow \R^+$ is a coverage function on some universe $U$, if and only if there exist non-negative coefficients $\alpha_S$ for every $S \subseteq [n],  S \neq \emptyset$ such that $c(x) = \sum_{S \subseteq [n], S \neq \emptyset} \alpha_S \cdot \ORR_S(x)$, and at most $|U|$ of the coefficients $\alpha_S$ are non-zero.
\label{disj-rep}
\end{lemma}
For simplicity and without loss of generality we scale coverage functions to the range $[0,1]$. Note that in this case, for $c = \sum_{S \subseteq [n], S \neq \emptyset} \alpha_S \cdot \ORR_S$ we have $\sum_{S \subseteq [n], S \neq \emptyset} \alpha_S = c((-1,\ldots,-1)) \leq 1$.
In the discussion below we always represent coverage functions as linear combination of monotone disjunctions with the sum of coefficients upper bounded by 1. For convenience, we also allow the empty disjunction (or constant 1) in the combination. Note that $\ORR_{[n]}$ differs from the constant 1 only on one point $(1,1,\ldots,1)$ and therefore this more general definition is essentially equivalent for the purposes of our discussion. Note that for every $S$, the  coefficient $\alpha_S$ is determined uniquely by the function since $\ORR_S$ is a monomial when viewed over $\zo^n$ with $0$ corresponding to ``true".
\subsection{Learning Models} Our learning algorithms are in several models based on the PAC model \citep{Val84}. In the PAC learning model the learner has access to random examples of an unknown function from a known class of functions and the goal is to output a hypothesis with low error. The PAC model was defined for Boolean functions with the probability of disagreement being used to measure the error. For our real-valued setting
we use $\ell_1$ error which generalizes the disagreement error.
\begin{definition}[PAC learning with $\ell_1$-error]
Let $\F$ be a class of real-valued functions on $\on^n$ and let $\D$ be a distribution on $\on^n$. An algorithm $\A$ PAC learns $\F$ on $\D$, if for every $\epsilon > 0$ and any target function $f \in \F$, given access to random  independent samples from $\D$ labeled by $f$, with probability at least $2/3$,  $\A$ returns a hypothesis $h$ such that $\E_{x \sim \D} [ |f (x) - h(x) | ] \leq  \epsilon$. $\A$ is said to be \emph{proper} if $h \in \F$.  $\A$ is said to be \emph{ efficient} if $h$ can be evaluated in polynomial time on any input and the running time of $\A$ is polynomial in $n$ and $1/\epsilon$.
\end{definition}
We also consider learning from random examples with multiplicative guarantees introduced by \citet{BalcanHarvey:12full} and referred to as PMAC learning. For a class of non-negative functions $\F$, a PMAC learner with approximation factor $\alpha \geq 1$ and error $\delta > 0$ is an algorithm which, with probability at least $2/3$, outputs a hypothesis $h$ that satisfies $\pr_{x \sim \D} [h (x) \leq f(x) \leq \alpha h(x) ] \geq  1-\delta $. We say that $h$ multiplicatively $(\alpha,\delta)$-approximates $f$ over $\D$ in this case.

We are primarily interested in the regime when the approximation ratio $\alpha$ is close to $1$ and hence use $1+\gamma$ instead. We say that the learner is {\em fully-polynomial} if it is polynomial in $n$,$1/\gamma$ and $1/\delta$.

\subsection{Fourier Analysis on the Boolean Cube}
When learning with respect to the uniform distribution we use several standard tools and ideas from Fourier analysis on the Boolean hypercube. For any functions $f,g: \on^n \rightarrow \R$, the inner product of $f$ and $g$ is defined as $\langle f, g \rangle = \E_{x \sim \U} [f(x) \cdot g(x)]$. The $\ell_1$ and $\ell_2$ norms of $f$ are defined by $\|f\|_1 =  \E_{x \sim \U} [|f(x)|]$ and $\|f\|_2^2 =  \E_{x \sim \U} [f(x)^2]$, respectively. Unless noted otherwise, in this context all expectations are with respect to $x$ chosen from the uniform distribution.

For $S \subseteq [n]$, the parity function $\chi_S:\on^n \rightarrow \on$ is defined as $ \chi_S(x) = \prod_{i \in S} x_i.$ Parities form an orthonormal basis for functions on $\on^n$ (for the inner product defined above). Thus, every function $f: \on^n \rightarrow \R$ can be written as a real linear combination of parities. The coefficients of the linear combination are referred to as the Fourier coefficients of $f$. For $f:\on^n \rightarrow \R$ and $S \subseteq [n]$, the Fourier coefficient $\hat{f}(S)$ is given by $\hat{f}(S) = \langle f, \chi_S \rangle = \E[f(x) \chi_S(x)].$%
The Fourier expansion of $f$ is given by $ f(x) = \sum_{S \subseteq [n]} \hat{f}(S) \chi_S(x).$
For any function $f$ on $\on^n$ its spectral $\ell_1$-norm is defined as $ \|\hat{f}\|_1 = \sum_{S \subseteq [n]} |\hat{f}(S)|.$

It is easy to estimate any Fourier coefficient of a function $f: \on^n \rightarrow [0,1]$, given access to an oracle that outputs the value of $f$ at a uniformly random point in the hypercube. Given any parameters $\epsilon, \delta > 0$, we choose a set $R \subseteq \on^n$ of size $\Theta(\log{\frac{1}{\delta}/\epsilon^2})$ drawn uniformly at random from $\{-1,1\}^n$ and estimate $\tilde{f}(S) = \frac{1}{|R|} \sum_{x \in R} [f(x) \cdot \chi_{S}(x)]$. Standard Chernoff bounds can then be used to show that with probability at least $1 - \delta$, $| \hat{f}(S) - \tilde{f}(S)| \leq \epsilon. $
For any $\epsilon >  0$, a Boolean function $f$ is said to be $\epsilon$-concentrated on a set $\S \subseteq 2^{[n]}$ of indices,  if $$\E \left[\left(f(x) -  \sum_{S \in \S} \hat{f}(S) \chi_S(x) \right)^2\right] = \sum_{S \notin \S} \hat{f}(S)^2 \leq \epsilon. $$

The following simple observation (implicit in \citep{KushilevitzMansour:93}) can be used to obtain spectral concentration from bounded spectral $\ell_1$-norm for any function $f$. In addition, it shows that approximating each large Fourier coefficient to a sufficiently small additive error yields a sparse linear combination of parities that approximates $f$. For completeness we include a proof in App.~\ref{app:proofs}.
\begin{lemma}
Let $f: \on^n \rightarrow \R$ be any function with $\|f\|_2 \leq 1$. For any $\epsilon \in (0,1]$, let $L=  \|\hat{f}\|_1$ and $\T = \{ T \mid| \hat{f}(T)| \geq \frac{\epsilon}{2L} \}$. Then $f$ is $\epsilon/2$-concentrated on $\T$ and $|\T| \leq \frac{2L^2}{\epsilon}$. Further, let $\S \supseteq \T$ and for each $S \in \S$, let $\tilde{f}(S)$ be an estimate of $\hat{f}(S)$ such that
\begin{enumerate}
\item $\forall S \in \S$, $|\tilde{f}(S)| \geq \frac{\epsilon}{3L}$ and
\item $\forall S \in \S$, $|\tilde{f}(S) - \hat{f}(S)| \leq \frac{\epsilon}{6L}$.
\end{enumerate}
Then, $\E[(f (x)- \sum_{S \in \S} \tilde{f}(S) \cdot \chi_S(x))^2] \leq \epsilon$ and, in particular, $\|f - \sum_{S \in \S} \tilde{f}(S) \cdot \chi_S \|_1 \leq \sqrt{\epsilon}$. \label{spectralconc}
\end{lemma}

\section{Learning Coverage Functions on the Uniform Distribution}
Here we present our PAC and PMAC learning algorithms for $\Cv$ over the uniform distribution.
 \label{sec:paclearning}
\subsection{Structural Results} We start by proving several structural lemmas about the Fourier spectrum of coverage functions. First, we observe that the spectral $\ell_1$-norm of coverage functions is at most $2$.
 \begin{lemma}
For a coverage function $c:\on^n \rightarrow [0,1]$, $\|\hat{c}\|_1 \leq 2$. \label{spectral-norm}
\end{lemma}
\begin{proof}
From Lem.~\ref{disj-rep} we have that there exist non-negative coefficients $\alpha_S$ for every $S \subseteq [n]$ such that $c(x) = \sum_{S \subseteq [n]} \alpha_S \cdot \ORR_S(x)$.
By triangle inequality, we have: $\|\hat{c}\|_1 \leq \sum_{S \subseteq [n]} \alpha_S \cdot \|\widehat{\ORR_S}\|_1 \leq \max_{S \subseteq [n]} \|\widehat{\ORR_S}\|_1  \cdot  \sum_{S \subseteq [n]} \alpha_S \leq \max_{S \subseteq [n]} \|\widehat{\ORR_S}\|_1 .$ To complete the proof, we verify that $\forall S \subseteq [n]$, $\|\widehat{\ORR_S}\|_1 \leq 2$. For this note that $\ORR_S(x) = 1 - \frac{1}{2^{|S|}}\cdot \Pi_{i \in S} (1+x_i) = 1 - \frac{1}{2^{|S|}} \sum_{T\subseteq S} \chi_T(x)$ and thus $\|\widehat{\ORR_S}\|_1 \leq 1+ \frac{1}{2^{|S|}} 2^{|S|} = 2$.
\end{proof}
The small spectral $\ell_1$-norm guarantees that any coverage function has its Fourier spectrum {$\epsilon^2$-concentrated} on some set $\T$ of indices of size $O(\frac{1}{\epsilon^2})$ (Lem.~\ref{spectralconc}). This means that given an efficient algorithm to find a set $\S$ of indices such that $\S$ is of size $O(\frac{1}{\epsilon^2})$ and $\S \supseteq \T$ we obtain a way to PAC learn coverage functions to $\ell_1$-error of $\epsilon$. In general, given only random examples labeled by a function $f$ that is concentrated on a small set $\T$ of indices, it is not known how to efficiently find a small set $\S \supseteq  \T$, without additional information about $\T$ (such as all indices in $\T$ being of small cardinality). However, for coverage functions, we can utilize a simple monotonicity property of their Fourier coefficients to efficiently retrieve such a set $\S$ and obtain a PAC learning algorithm with running time that depends only polynomially on $1/\epsilon$.
\begin{lemma}[Lem.~\ref{lem:monotonicity} restated]
Let $c: \on^n \rightarrow [0,1]$ be a coverage function. For any non empty $T \subseteq V \subseteq [n]$, $ |\hat{c}(V)| \leq |\hat{c}(T)| \leq \frac{1}{2^{|T|}}$.
\label{monotonicity}
\end{lemma}
\begin{proof}
From Lem.~\ref{disj-rep} we have that there exist constants $\alpha_S \geq 0$ for every $S \subseteq [n]$ such that $\sum_{S \subseteq [n]} \alpha_S \leq 1$ and $c(x) = \sum_{S \subseteq [n]} \alpha_S \ORR_S(x)$ for every $x \in \on^n$. The Fourier transform of $c$ can now be obtained simply by observing, as before in Lem.~\ref{spectral-norm} that $\ORR_S(x) = 1 - \frac{1}{2^{|S|}} \sum_{T\subseteq S} \chi_T(x)$.
Thus for every $T \neq \emptyset$, $\hat{c}(T) = - \sum_{ S \supseteq T} \alpha_S \cdot (\frac{1}{2^{|S|}}).$ Notice that since all the coefficients $\alpha_S$ are non-negative, $\hat{c}(T)$ and $\hat{c}(V)$ are non-positive and
$|\hat{c}(T)| = \sum_{ S \supseteq T} \alpha_S \cdot (\frac{1}{2^{|S|}}) \geq \sum_{ S \supseteq V} \alpha_S \cdot (\frac{1}{2^{|S|}}) = |\hat{c}(V)|\ .$ For an upper bound on the magnitude $|\hat{c}(T)|$, we have:
$|\hat{c}(T) | = \sum_{S \supseteq T} \alpha_S \cdot \frac{1}{2^{|S|}} \leq \sum_{S \supseteq T} \alpha_S \cdot \frac{1}{2^{|T|}}  \leq (\sum_{S \subseteq [n]} \alpha_S) \cdot \frac{1}{2^{|T|} } \leq \frac{1}{2^{|T|}}.$
\end{proof}
We will now use Lemmas \ref{spectral-norm} and \ref{monotonicity} to show that for any coverage function $c$, there exists another coverage function $c'$ that depends on just $O(1/\epsilon^2)$ variables and $\ell_1$-approximates it within $\epsilon$. Using Lem.~\ref{spectral-norm}, we also obtain spectral concentration for $c$. We start with some notation: for any $x \in \on^n$ and a subset $J \subseteq [n]$ of variables, let $x_J \in \on^J$ denote the projection of $x$ on $J$. Given $y \in \on^J$ and $z \in \on^{\bar{J}}$, let $x = y \circ z$ denote the string in $\on^n$ such that $x_J = y$ and $ x_{\bar{J}}= z$ (where $\bar{J}$ denotes the set $[n]\setminus J$). We will need the following simple lemma that expresses the Fourier coefficients of the function $f_I$ which is obtained by averaging a function $f$ over all variables outside of $I$ (a proof can be found for example in \citep{KushilevitzMansour:93}).
\begin{lemma}
For $f: \on^n \rightarrow [0,1]$ and $I \subseteq [n]$, let $f_I(x) = \E_{y \sim \on^{\bar{I}}} [   f(x_{I} \circ y)  ]$. Then, $\hat{f_I}(S) = \hat{f}(S)$ for every $S \subseteq \I$ and $\hat{f_I}(T) = 0$ for every $T \nsubseteq \I$.
\label{Fourieravg}
\end{lemma}
\eat{
\begin{proof}
Observe that by the definition of $f_I$, for any $z \in \on^\I$ and any $y_1, y_2 \in \on^{\bar{\I}}$, $f_I(z \circ y_1) = f_I(z \circ y_2)$ and thus
\begin{equation}
\E_{y \in \on^{\bar{\I}}} f_I(z \circ y) = \E_{y \in \on^{\bar{\I}}} f(z \circ y). \label{st:ff}\end{equation}
Let $S \subseteq \I$. Then,

\begin{align*}
\hat{f}(S) &= \E_{x \sim \on^n} [ f(x) \cdot \chi_S(x)]  = \E_{z \sim \on^\I, \text{  } y \sim \on^{\bar{\I}}} [f(z \circ y) \chi_S(z \circ y)]\\
&= \E_{z \sim \on^\I}  [ \Pi_{i\in S} z_i \cdot \E_{y \sim \on^{\bar{\I}}} [f(z \circ y)] ]\\
\text{Using} & \text{ Equation \eqref{st:ff}}\\
&= \E_{z \sim \on^\I}  [ \Pi_{i\in S} z_i \cdot \E_{y \sim \on^{\bar{\I}}} [f_I(z \circ y)] ]
= \E_{z \sim \on^\I} \E_{y \sim \on^{\bar{\I}}}[   \Pi_{i \in S} z_i  \cdot f_I(z \circ y)       ]\\
&= \E_{x \sim \on^n} [\chi_S(x) \cdot f_I(x)] = \hat{f_I}(S).
\end{align*}

Now, let $T = V \cup W$ for $V \subseteq I$ and a non-empty $W \subseteq \bar{I}$.
\begin{align*}
\hat{f_I}(T) &= \E_{x \sim \on^n} [ f_I(x) \cdot \chi_T(x)]  = \E_{z \sim \on^\I, \text{  } y \sim \on^{\bar{\I}}} [f_I(z \circ y) \chi_T(z \circ y)]\\
&= \E_{z \sim \on^\I}  [ \Pi_{i\in V} z_i \cdot \E_{y \sim \on^{\bar{\I}}} [f_I(z \circ y) \cdot \Pi_{i \in W} y_i] ].
\end{align*}

But $f_I(z \circ y)$ does not depend on $y$ and $\E_{y \sim \on^{\bar{\I}}} [ \Pi_{i \in W} y_i] = 0$. Thus, $\hat{f_I}(T) = 0$.
\end{proof}
}
We now show that coverage functions can be approximated by functions of few variables.
\begin{theorem}[Thm.~\ref{thm:junta-coverageintro} restated]
Let $c:\on^n \rightarrow [0,1]$ be a coverage function and $\epsilon>0$. Let $\I = \{ i \in [n] \mid |\hat{c}(\{i\})| \geq \frac{\epsilon^2}{2}\}$. Let $c_I$ be defined as $c_I(x) = \E_{y \sim \on^{\bar{I}} }[ c(x_I \circ y)]$. Then $c_I$ is a coverage function that depends only on variables in $I$, $|I| \leq 4/\epsilon^2$, $\size(c_I) \leq \size(c)$ and $\|c-c_I\|_1 \leq \epsilon$. Further, let $\T = \{ T \subseteq [n] \mid |\hat{c}(T)| \geq \frac{\epsilon^2}{2}\}$. Then $\T \subseteq 2^{\I}$ and $c$ is $\epsilon^2$-concentrated on $\T$.
\label{junta-coverage}
\end{theorem}
\begin{proof}
 Since $c$ is a coverage function, it can be written as a non-negative weighted sum of monotone disjunctions. Thus, for every $v \in \on^{\bar{I}}$ the function, $c_v:\on^n \rightarrow [0,1]$ defined as $c_v(z \circ y) = c(z \circ v)$ for every $y \in \on^{\bar{I}}$ is also a non-negative linear combination of monotone disjunctions, that is a coverage function.
By definition, for every $z \in \on^{\I}$ and $y \in \on^{\bar{I}}$, $c_I(z \circ y) =\frac{1}{2^{n-|I|}} \sum_{v \in \on^{\bar{I}}} c(z \circ v) =\frac{1}{2^{n-|I|}} \sum_{v \in \on^{\bar{I}}} c_v(z \circ y) .$ In other words, $c_I$ is a convex combination of $c_v$'s and therefore is a coverage function itself. Note that for every $S \subseteq I$ if the coefficient of $\ORR_S$ in $c_I$ is non-zero then there must exist $S' \subseteq \bar{I}$ for which the coefficient of $\ORR_{S \cup S'}$ in $c$ is non-zero. This implies that $\size(c_I) \leq \size(c)$. We will now establish that $c_I$ approximates $c$.
Using Lem.~\ref{Fourieravg}, $\hat{c}(S) = \hat{c}_I(S)$ for every $S \subseteq I$. Thus, $\|c - c_I\|_2^2 = \sum_{T \nsubseteq \I} \hat{c}(T)^2$. We first observe that $\T \subseteq 2^I$. To see this, consider any $T \nsubseteq \I$. Then, $\exists i \not\in \I$ such that $i \in T$ and therefore, by Lem.~\ref{monotonicity}, $|\hat{c}(T)| \leq |\hat{c}(\{i\}) | < \eps^2/2$. Thus $|\hat{c}(T)| \leq \eps^2/2$. By Lem.~\ref{spectralconc}, $c$ is $\eps^2$-concentrated on $\T$ and using Cauchy-Schwartz inequality, $\|c - c_I\|_1^2 \leq \|c - c_I\|_2^2 = \sum_{T \nsubseteq \I} \hat{c}(T)^2 \leq \sum_{T \not\in \T} \hat{c}(T)^2 \leq \eps^2.$
\end{proof}
\subsection{PAC Learning} \label{sec:imppaclearning}
We now describe our PAC learning algorithm for coverage functions. This algorithm is used for our application to private query release and also as a subroutine for our PMAC learning algorithm. Given the structural results above the algorithm itself is quite simple. Using random examples of the target coverage function, we compute all the singleton Fourier coefficients and isolate the set $\tilde{\I}$ of coordinates corresponding to large (estimated) singleton coefficients that includes $\I = \{ i \in [n] \mid |\hat{c}(\{i\})| \geq \frac{\epsilon^2}{4}\}$. Thm.~\ref{junta-coverage} guarantees that the target coverage function is concentrated on the large Fourier coefficients, the indices of which are subsets of $\tilde{I}$. We then find a collection $\S \subseteq 2^{\tilde{\I}}$ of indices that contains all $T \subseteq \tilde{\I}$ such that $|\hat{c}(T) | \geq \epsilon^2/4$. This can be done efficiently since by Lem.~\ref{monotonicity}, $|\hat{c}(T) | \geq\epsilon^2/4$ only if $|\hat{c}(V) | \geq\epsilon^2/4$ for all $V\subseteq T$, $V \neq \emptyset$. We can only estimate Fourier coefficients up to some additive error with high probability and therefore we keep all coefficients in the set $\S$ whose estimated magnitude is at least $\epsilon^2/6$. Once we have a set $\S$, on which the target function is $\epsilon^2$-concentrated, we use Lem.~\ref{spectralconc} to get our hypothesis.
We give the pseudocode of the algorithm below.
\begin{algorithm}
\caption{PAC Learning of Coverage Functions}
\begin{algorithmic}[1]
\STATE Set $\theta = \frac{\epsilon^2}{6}$.
\STATE Draw a random sample of size $m_1=O(\log{(n)}/\eps^4)$ and use it to
estimate $\hat{c}(\{i\})$ for all $i$.
\STATE Set $\tilde{\I} = \{ i \in [n] \mid |\tilde{c}(\{i\})| \geq \theta\}$.
\STATE $\S \leftarrow \{\emptyset\}$.
\STATE Draw random sample $R$ of size $m_2 = O(\log{(1/\eps)}/\eps^4)$.
\FOR{$t=1$ to $\log(2/\theta)$}
\FOR{each set $T \in \S$ of size $t-1$ and $i\in \tilde{\I}\setminus T$ }
\STATE Use $R$ to estimate the coefficient $\hat{c}({T \cup \{i\}})$.
\STATE If $|\tilde{c}({T \cup i}) | \geq \theta$ then $\S \leftarrow \S \cup \{T \cup \{i\}\}$
\ENDFOR
\ENDFOR
\RETURN $\sum_{S \in \S} \tilde{c}(S) \cdot \chi_S$.
\end{algorithmic}
\label{PAC learn-coverage}
\end{algorithm}

\begin{theorem}[Thm.~\ref{thm:pacintro} restated]
There exists an algorithm that PAC learns $\Cv$ in  $\tilde{O}(n/\epsilon^4 + 1/\eps^8)$ time and using $\log{n} \cdot \tilde{O}(1/\epsilon^4)$ examples.
\label{thm:paccov}
\end{theorem}
\begin{proof}
Let $c$ be the target coverage function and let  $\T = \{ T \subseteq [n] \mid |\hat{c}(T)| \geq \frac{\epsilon^2}{4} \}$. By Lem.~\ref{spectralconc}, it is sufficient to find a set $\S \supseteq \T$ and estimates $\tilde{c}(S)$ for each $S \in \S$ such that:
\begin{enumerate}
\item $\forall S \in \S$  $|\tilde{c}(S)| \geq \frac{\epsilon^2}{6}$ and
\item $\forall S \in \S$, $|\tilde{c}(S) - \hat{c}(S)| \leq \frac{\epsilon^2}{12}$.
\end{enumerate}

Let $\theta = \eps^2/6$. In the first stage our algorithm finds a set $\tilde{I}$ of variables that contains $\I = \{ i \in [n] \mid |\hat{c}(\{i\})| \geq \frac{\epsilon^2}{4}\}$. We do this by estimating all the singleton Fourier coefficients, $\{ \hat{c}(\{i\}) \mid i \in [n] \}$ within $\theta/2$ with (overall) probability at least $5/6$ (as before we denote the estimate of $\hat{c}(S)$ by $\tilde{c}(S)$). We set $\tilde{\I} =\{  i \in [n] \cond \tilde{c}(\{i\})| \geq \theta\}$.
If all the estimates are within $\theta/2$ of the corresponding coefficients then for every $i \in \I$, $\tilde{c}(\{i\}) \geq \eps^2/4 - \theta/2 = \eps^2/6=\theta$. Therefore $i\in \tilde{\I}$ and hence $\I \subseteq \tilde{\I}$.

In the second phase, the algorithm finds a set $\S \subseteq 2^{\tilde{\I}}$ such that the set of all large Fourier coefficients $\T$ is included in $\S$. This is done iteratively starting with $\S = \{\emptyset\}$. In every iteration, for every set $T$ that was added in the previous iteration and every $i \in \tilde{\I}\setminus T$, it estimates $\hat{c}(T \cup \{i\})$ within $\theta/2$ (the success probability for estimates in this whole phase will be $5/6$). If $|\tilde{c}(T \cup \{i\})| \geq \theta$ then $T \cup \{i\}$ is added to $\S$. This iterative process runs until no sets are added in an iteration. At the end of the last iteration, the algorithm returns $\sum_{S \in \S} \tilde{c}(S) \chi_S$ as the hypothesis.

We first prove the correctness of the algorithm assuming that all the estimates are successful. Let $T \in \T$ be such that $|\hat{c}(T)| \geq \eps^2/4$. Then, by Thm.~\ref{junta-coverage}, $T \subseteq \I \subseteq \tilde{\I}$. In addition, by Lem.~\ref{monotonicity}, for all $V\subseteq T$, $V \neq \emptyset$,
$|\hat{c}(V) | \geq\epsilon^2/4$. This means that for all $V\subseteq T$, $V \neq \emptyset$ an estimate of $|\hat{c}(V)|$ within $\theta/2$ will be at least $\theta$. By induction on $t$ this implies that in iteration $t$, all subsets of $T$ of size $t$ will be added to $\S$ and $T$ will be added in iteration $|T|$. Hence the algorithm outputs a set $\S$ such that $\T \subseteq \S$. By definition, $\forall S \in \S$,  $|\tilde{c}(S)| \geq \theta = \frac{\epsilon^2}{6}$ and $\forall S \in \S$, $|\tilde{c}(S) - \hat{c}(S)| \leq \theta/2 = \frac{\epsilon^2}{12}$. By Lem.~\ref{spectralconc}, $\|c - \sum_{S \in \S} \tilde{c}(S) \chi_S\|_1 \leq \eps$.

We now analyze the running time and sample complexity of the algorithm. We make the following observations regarding the algorithm.
\begin{itemize}
	\item By Chernoff bounds, $O(\log {(n)}/\theta^2) = O(\log{(n)} /\epsilon^4)$ examples suffice to estimate all singleton coefficients within $\theta/2$ with probability at least $5/6$. To estimate a singleton coefficients of $c$, the algorithm needs to look at only one coordinate and the label of a random example. Thus all the singleton coefficients can be estimated in time $O(n \log{(n)} /\epsilon^4)$.
    \item For every $S$ such that $\hat{c}(S)$ was estimated within $\theta/2$ and $|\tilde{c}(S)| \geq \theta$, we have that $|\hat{c}(S)| \geq \theta/2 = \eps^2/12$. This implies that $|\tilde{\I}| \leq 2/(\theta/2) = 24/\eps^2$.
        This also implies that $|\S| \leq 4/\theta = 24/\eps^2$.
	\item By Lem.~\ref{monotonicity},  for any $T \subseteq [n]$, $|\hat{c}(T)| \leq \frac{1}{2^{|T|}}$. Thus, if $|\hat{c}(T)| \geq \theta/2$ then $|T| \leq  \log{(2/\theta)}$. This means that the number of iterations in the second phase is bounded by $\log{(2/\theta)}$ and for all $S \in \S$, $|S| \leq \log{(2/\theta)}$.
	\item In the second phase, the algorithm only estimates coefficients for subsets in $$\S' = \{ S \cup \{i\} \cond |\tilde{c}(S)| \geq \theta \mbox{ and } i \in \tilde{\I}\}. $$
Let $\T' = \{ T \cup \{i\} \cond |\hat{c}(T)|\geq \theta/2 \mbox{ and } i \in \tilde{\I}\}$. By Chernoff bounds, a random sample of size $O(\log{|\T'|}/\theta^2)=\tilde{O}(1/\eps^4)$ can be used to ensure that, with probability at least $5/6$, the estimates of all coefficients on subsets in $\T'$ are within $\theta/2$. When the estimates are successful we also know that $\S' \subseteq \T'$ and therefore all coefficients estimated by the algorithm in the second phase are also within $\theta/2$ of true values with probability $\geq 5/6$.
Overall in the second phase the algorithm estimates $|\S'| \leq |\S| \cdot |\tilde{\I}| = O(1/\eps^4)$  coefficients. To estimate any single of those coefficients, the algorithm needs to examine only $\log{(2/\theta)} = O(\log{(1/\epsilon)})$ coordinates and the label of an example. Thus, the estimation of each Fourier coefficient takes $\tilde{O}(1/\epsilon^4)$ time and $\tilde{O}(1/\epsilon^8)$ time is sufficient to estimate all the coefficients.
\end{itemize}
Thus, in total the algorithm runs in $\tilde{O}(n/\epsilon^4 + 1/\eps^8)$ time, uses $\log n \cdot \tilde{O}(1/\eps^4)$ random examples and succeeds with probability at least $2/3$.
\end{proof}

\subsection{PMAC Learning} We now describe our PMAC learning algorithm that is based on a reduction from multiplicative to additive approximation. First we note that if we knew that the values of the target coverage function $c$ are lower bounded by some $m > 0$ then we could obtain multiplicative $(1+\gamma, \delta)$-approximation using a hypothesis $h$ with $\ell_1$ error of $\gamma \delta m/2$. To see this note that, by Markov's inequality, $\E[|h(x) - c(x)|] \leq \gamma \delta m/2$ implies that $\pr[|h(x) - c(x)| > \gamma m/2] \leq \delta$. Let $h'(x) = \max\{m, h(x) - \gamma m /2\}$. Then \alequ{1-\delta & \leq \pr[|h(x) - c(x)| \leq \gamma m/2] = \pr[h(x)-\gamma m/2 \leq c(x) \leq h(x)+\gamma m/2] \nonumber\\&\leq \pr[h'(x) \leq c(x) \leq h'(x)+\gamma m]  \leq \pr[h'(x) \leq c(x) \leq (1+\gamma) h'(x)] \label{eq:add-to-mult}
}
Now, we might not have such a lower bound on the value of $c$. To make this idea work for all coverage functions, we show that any monotone submodular function can be decomposed into regions where it is relatively large (compared to the maximum in that region) with high probability. The decomposition is based on the following lemma: given a monotone submodular function $f$ with maximum value $M$, either $\pr[f(x) \geq M/4] \geq 1-\delta/2$ or there is an index $i\in [n]$ such that $f(x) \geq \frac{M}{16 \ln{(2/\delta)}}$ for every $x$ satisfying $x_i = -1$. In the first case we can obtain multiplicative approximation from additive approximation using a slight refinement of our observation above (since the lower bound on $f$ only holds with probability $1-\delta/2$). In the second case we can reduce the problem to additive approximation on the half of the domain where $x_i = -1$. For the other half we use the same argument recursively. After $\lceil\log{(2/\delta)})\rceil$ levels of recursion at most $\delta/2$ fraction of the points will remain where we have no approximation. Those are included in the probability of error.
We will need the following concentration inequality for 1-Lipschitz (with respect to the Hamming distance) submodular functions \citep{BLM00,Von10,BalcanHarvey:12full}.
\begin{theorem}[\citealp{Von10}]
For a non-negative, monotone, 1-Lipschitz submodular function $f$ and $0 \leq \beta < 1$, $\pr_\U[ f(x) \leq (1-\beta) \E_\U[f(x)]] \leq e^{-\beta^2\E[f]/2}.$ \label{thm:concsubmodular}
\end{theorem}
Another property of non-negative monotone submodular functions that we need is that their expectation is at least half the maximum value \citep{Feige:06}. For the special case of coverage functions this lemma follows simply from the fact that the expectation of any disjunction is at least $1/2$.
\begin{lemma}[\citealp{Feige:06}]
\label{lem:expect}
For $f$, a non-negative monotone submodular function, $\E_\U[f] \geq \|f\|_\infty/2$.
\end{lemma}
We now prove our lemma that lower bounds the relative value of a monotone submodular function.
\begin{lemma}
Let $f$ be a non-negative monotone submodular function and $M=\|f\|_\infty$. Then for every $\delta > 0$, either $\pr_\U [f(x) \leq M/4] \leq \delta$ or there exists an $i\in[n]$ such that $f(x)\geq \frac{M}{16\ln{(1/\delta)}}$ for every $x$ such that $x_i = -1$. \label{lem:large-labels}
\end{lemma}
\begin{proof}
Let $e_i \in \on^n$ equal the bit string that has $-1$ in its $i^{th}$ coordinate and $1$ everywhere else. For any $x \in \on^n$ let $x \oplus y$ denote the string $z$ such that $z_j = x_j \cdot y_j$ for every $j \in [n]$. Suppose that for every $i \in [n]$, exists $x$ such that $x_i=-1$ and $f(x) \leq \frac{M}{16\ln{(1/\delta)}}$. By monotonicity of $f$ that implies that for every $i \in [n]$, $f(e_i) \leq \frac{M}{16\ln{(1/\delta)}}$. Since $f$ is a submodular function, for any $x$ and $i$ such that $x_i = 1$, we have: $f(x \oplus e_i) - f(x) \leq f(1^n \oplus e_i) - f(1^n) \leq f(e_i) \leq \frac{M}{16\ln{(1/\delta)}}.$ This implies that $f$ is $\frac{M}{16\ln{(1/\delta)}}$-Lipschitz. Then, $f' = f/\frac{M}{16 \ln{(1/\delta)}}$ is a $1$-Lipschitz, non-negative submodular function. Also, by Lem.~\ref{lem:expect}, $\E[f] \geq M/2$ and $\E[f'] \geq 8 \ln{(1/\delta)}$.
Now, using Thm.~\ref{thm:concsubmodular}, we obtain: $\pr[ f(x) \leq M/4 ] \leq \pr[ f'(x) \leq \frac{1}{2} \E[f'] ] \leq e^{-\frac{1}{8}\E[f']} \leq e^{-\ln{(1/\delta)}} =  \delta.$
\end{proof}

Recall that for any set $J \subseteq [n]$ of variables and $x \in \on^n$, $x_J \in \on^J$ is defined as the substring of $x$ that contains the bits in coordinates indexed by $J$. We are now ready to describe our reduction that gives a PMAC algorithm for coverage functions.
\begin{theorem}[Thm.~\ref{thm:PMACintro} restated]
There exists an algorithm $\A$ which, given $\gamma,\delta > 0$ and access to random uniform examples of any coverage function $c$, with probability at least $2/3$, outputs a hypothesis $h$ such that $\pr_\U[ h(x) \leq c(x) \leq(1+\gamma) h(x)] \geq 1-\delta$. Further, $\A$ runs in  $\tilde{O}(\frac{n}{\gamma^4 \delta^4} + \frac{1}{\gamma^8\delta^8})$  time and uses $\log{n} \cdot \tilde{O} (\frac{1}{\gamma^4 \delta^4})$ examples.
\end{theorem}

 \begin{proof}
Algorithm $\A$ consists of a call to $\A'(0)$, where $\A'(k)$ is a recursive procedure described below.

\smallskip
\noindent \textbf{Procedure $\A'(k)$ on examples labeled by $c: \on^n \rightarrow \R^+$}:
\begin{enumerate}

\item If $k > \log{(3/\delta)}$, then, $\A'(k)$ returns the hypothesis $h \equiv 0$ and halts.

\item Otherwise, $\A'(k)$ computes a $3$-approximation to the maximum $M$ of the target function $c$ (with confidence at least $1-\eta$ for $\eta$ to be defined later). As we show later this can be done by drawing a sufficient number of random examples labeled by $c$ and choosing $\tilde{M}$ to be the maximum label. Thus, $\frac{M}{3} \leq \tilde{M} \leq M$. If $\tilde{M} = 0$, {\bf return} $h \equiv 0$. Otherwise, set $c' = \frac{c}{3\tilde{M}}$ (note that, with probability at least $1-\eta$, $c'(x) \in [0,1]$ for every $x$).

\item Estimate $p = \pr[c(x)\leq \tilde{M}/4]$ within an additive error of $\frac{\delta}{9}$ by $\tilde{p}$ with confidence at least $1-\eta$. Then, $p - \frac{\delta}{9} \leq \tilde{p} \leq p+\frac{\delta}{9}$.

\item If $\tilde{p} < 2\delta/9$: run Algorithm from Thm. \ref{thm:paccov} on random examples labeled by $c'$ with accuracy $\epsilon_1 = \frac{1}{12}\frac{\gamma}{2}\frac{\delta}{3}$ and confidence $1-\eta$ (note that Algorithm from Thm. \ref{thm:paccov} only gives $2/3$ confidence but the confidence can be boosted to $1-\eta$ using $O(\log(1/\eta))$ repetitions with standard hypothesis testing). Let $h'$ be the hypothesis output by the algorithm. {\bf Return} hypothesis $h = \max\{\tilde{M}/4, 3 \tilde{M} (h' - \gamma/24)\} $.

\item If $\tilde{p} \geq 2\delta/9$,

\begin{enumerate}
\item Find $j \in [n]$ such that $c(x) \geq \tilde{M}/(16 \ln{(9/\delta)})$ for every $x$ such that $x_j = -1$ with confidence at least $1-\eta$. This can be done by drawing a sufficient number of random examples and checking the labels. If such $j$ does not exist
    we output $h\equiv 0$. Otherwise, define $c_{j,-}:\on^{[n]\setminus {j}} \rightarrow \R^+$ to be the restriction of $c$ to $\on^{[n]\setminus {j}}$ where $x_j = -1$ and $c'_{j,-} = c_{j,-}/(3\tilde{M})$.
    Run the algorithm from Thm. \ref{thm:paccov} on examples labeled by $c'_{j,-}$ with accuracy $\epsilon' = \frac{\gamma}{2}\cdot \frac{\delta}{3} \cdot  \frac{1}{48\ln{(9/\delta)}}$ and confidence $1-\eta$. Let $h'_-$ be the hypothesis returned by the algorithm. Set $h_- = \max\{\tilde{M}/4, 3 \tilde{M} (h'_- - \frac{\gamma}{96\ln{(9/\delta)}})\}$.

\item Let $c_{j,+}:\on^{[n]\setminus {j}} \rightarrow \R^+$ to be the restriction of $c$ to $\on^{[n]\setminus {j}}$ where $x_j = +1$. Run $\A'(k+1)$ on examples labeled by $c_{j,+}$ and let $h_+$ be the hypothesis returned by the algorithm.
\item {\bf Return} hypothesis $h:\on^n \rightarrow \R^+$ defined by $h(x) = \left \{ \begin{subarray}[ h_-(x_{[n]\setminus j}) \text{  if  } x_j = -1\\
                                                                                         h_+(x_{[n]\setminus j}) \text{ if } x_j =  1 \\
                                                                                         \end{subarray}
                                                                                         \right.$
\end{enumerate}
 \end{enumerate}
The algorithm can simulate random examples labeled by $c'_{j,-}$ (or $c_{j,+}$) by drawing random examples labeled by $c$, selecting $(x,\ell)$ such that $x_j = -1$ (or $x_j = 1$) and removing the $j$-th coordinate. Since $k \leq \log{(3/\delta)}$ bits will need to be fixed the expected number of random examples required to simulate one example from any function in the run of $\A(k)$ is at most $3/\delta$.

We now prove the correctness of the algorithm assuming that all random estimations and runs of the PAC learning algorithm are successful.
To see that one can estimate the maximum $M$ of a coverage function $c$ within a multiplicative factor of $3$, recall that by Lem.~\ref{lem:expect}, $\E[c]\geq M/2$. Thus, for a randomly and uniformly chosen $x \in \on^n$, with probability at least $1/4$, $c(x) \geq M/3$. This means that $\log(2/\eta)$ random examples will suffice to get confidence $1-\eta$.

We now observe that if the condition in step $4$ holds then $h$ $(1+\gamma,2\delta/3)$-multiplicatively approximates $c$. To see this, first note that in this case, $p =  \pr[c(x)\leq \tilde{M}/4] \leq \tilde{p} +\delta/9 \leq \delta/3$. Then, $\pr[ c'(x) \leq (\tilde{M}/4)/(3\tilde{M})] \leq \delta/3$. By Thm.~\ref{thm:paccov}, $\E[|c'(x)-h'(x)|] \leq \frac{1}{12} \frac{\gamma}{2}\cdot \frac{\delta}{3}$.
Then, by Markov's inequality,
$$\pr[h'(x) - \gamma/24 > c'(x) \mbox { or } c'(x) > h'(x) + \gamma/24] \leq \delta/3 .$$
Let
$h''(x) = \max\{1/12, h'(x) - \gamma/24\}$.  By the same argument as in eq.~(\ref{eq:add-to-mult}), we get that
$$\pr[c'(x)\geq 1/12 \mbox{ and } (h''(x) > c'(x)  \mbox { or } c'(x) > (1+\gamma)h''(x))] \leq \delta/3 .$$ Therefore,
$$\pr[h''(x) \leq c'(x) \leq  (1+\gamma)h''(x)] \geq 1-2\delta/3 $$ or, equivalently,
$$\pr[h(x) \leq c(x) \leq  (1+\gamma)h(x)] \geq 1-2\delta/3 .$$

If the condition in step $4$ does not hold, then $p \geq 2\delta/9-\delta/9 = \delta/9$. Thus, $\pr[c(x)\leq M/4] \geq \pr[c(x)\leq \tilde{M}/4] \geq \delta/9$, which by Lem.~\ref{lem:large-labels} yields that there exists $j \in [n]$ such that $c(x) \geq M/(16 \ln{(9/\delta)})$. Now,
by drawing $O(\log(n/\eta)/\delta)$ examples and choosing $j$ such that for all examples where $x_j=-1$, $c(x) \geq M/(16 \ln{(9/\delta)})$ we can ensure that, with probability at least $1-\eta$, $$\pr_{y \in \on^{[n]\setminus{j}}}[c_{j,-}(y) \leq M/(16 \ln{(9/\delta)})] \leq \delta/3.$$
Now, by the same analysis as in step 4, we obtain that $h_-$ satisfies $\pr[h_- \leq c_{j,-} \leq (1+\gamma) h_-] \geq 1- 2\delta/3$.

Now, observe that the set of points in the domain $\on^n$ can be partitioned into two disjoint sets.
\begin{enumerate}
	\item The set $G$ such that for every $z \in G$, $\A$ has fixed the value of the hypothesis given by $h(z)$ based on some hypothesis returned by the PAC learning algorithm (Thm. \ref{thm:paccov}) or $h(z) \equiv 0$ when $\tilde{M} = 0$.
	\item The set $\bar{G}$ where the recursion has reached depth $k > \log{(3/\delta)}$ and step $1$ sets $h(x) \equiv 0$ on every point in $\bar{G}$.
\end{enumerate}

By the construction, the points in $G$ can be divided into disjoint sub-cubes such that in each of them, the
conditional probability that the hypothesis we output does not satisfy the multiplicative guarantee is at most $2\delta/3$.
Therefore, the hypothesis $h$ does not satisfy the multiplicative guarantee on at most $2\delta/3$ fraction of the points in $G$.
It is easy to see that $\bar{G}$ has probability mass at most $\delta/3$. This is because $\A = \A(0)$ and thus, when $k > \log{(3/\delta)}$, the dimension of the subcube that $\A'(k)$ is invoked on, is at most $n-\log{(3/\delta)}$. Thus, the total probability mass of points where the multiplicative approximation does not hold is at most $\delta$.

We now bound the running time and sample complexity of the algorithm. First note that for some $\eta  = O(1/\log{(1/\delta)})$ all the random estimations and runs of the PAC learning algorithm will be successful with probability at least $2/3$ (by union bound).

From Thm.~\ref{thm:paccov}, any run of the PAC learning algorithm in some recursive call to $\A'$ requires at most $\log{n} \cdot \log{(\frac{1}{\eta})} \cdot \tilde{O}( \frac{1}{ \gamma^4 \cdot \delta^4 })$ examples from their respective target functions. Each such example can be simulated using $\Theta(1/\delta)$ examples labeled by $c$. Thus, in total, in all recursive calls, $\log{n} \cdot \tilde{O}( \frac{1}{ \gamma^4 \cdot \delta^5 })$ examples will suffice.

Each run of the PAC learning algorithm requires $\tilde{O}(\frac{n}{\gamma^4 \delta^4} + \frac{1}{\gamma^8 \delta^8})$ time. The rest of the computations in any one recursive call to $\A'$ can be performed in time linear in the number of examples. Thus, total time required for an execution of $\A$ is bounded by $\tilde{O}(\frac{n}{\gamma^4 \delta^4} + \frac{1}{\gamma^8 \delta^8})$.

\end{proof}
\subsection{Proper PAC Learning Algorithm} \label{app:properpac}
We now present a PAC learning algorithm for coverage functions that guarantees that the returned hypothesis is also a coverage function. That is, the algorithm is \emph{proper}. The running time of the algorithm will depend polynomially on the size of the target coverage function. We will need a variant of linear regression with $\ell_1$ error for this algorithm which we now define formally.

\subsubsection{$\ell_1$ Linear Regression}
It is easy to see that given a set of $t$ examples $\{(x^i,y^i)\}_{i\leq t}$ and a set of $m$ functions $\phi_1, \phi_2, \ldots, \phi_m$ finding coefficients $\alpha_1,\ldots,\alpha_m$ which minimize $$\sum_{i\leq t} \left| \sum_{ j \leq m} \alpha_j \phi_j (x^i) - y^i \right|$$ can be formulated as a linear program. This LP is referred to as Least-Absolute-Error (LAE) LP or Least-Absolute-Deviation LP, or $\ell_1$ linear regression \citep{Wiki:LAD}. Together with standard uniform convergence bounds for linear functions \citep{Vap98}, $\ell_1$ linear regression gives a general technique for learning with $\ell_1$-error.
\begin{theorem}
\label{th:lae-lp}
Let $\F$ be a class of real-valued functions from $\on^n$ to $[-B,B]$ for some $B >0$, $\D$ be distribution on $\on^n$ and $\phi_1, \phi_2, \ldots, \phi_m: \on^n \rightarrow \R$ be a set of functions that can be evaluated in time polynomial in $n$. Assume that there exists $\Delta$ 
such that for each $f \in \F$, there exist reals $\alpha_1, \alpha_2, \ldots, \alpha_m$ such that $${\E_{x \sim \D} \left[  \left|\sum_{ i \leq m} \alpha_i \phi_i (x) - f(x)\right|\right] \leq \Delta}.$$
Then there is an algorithm that for every $\eps >0$ and any distribution $\P$ on $\on^n \times [0,1]$ such that the marginal of $\P$ on $\on^n$ is $\D$, given access to random samples from $\P$, with probability at least $2/3$, outputs a function $h$ such that $\E_{(x,y) \sim \P} [ |h(x)- y| ] \leq \Delta + \epsilon.$ The algorithm uses $O(m \cdot B^2/\eps^2)$ examples, runs in time polynomial in $n$, $m$, $B/\eps$ and returns a linear combination of $\phi_i$'s.
\end{theorem}
\begin{remark}
\label{rem:non-neg-lae}
Additional linear constraints on $\alpha_i$'s can be added to this LP as long as these constraints are satisfied by linear combinations that approximate each $f \in \F$ within $\Delta$. In particular we will use this LP with each $\alpha_i$ being constrained to be non-negative and their sum being at most 1.
\end{remark}
We remark that this approach to learning can equivalently be seen as learning based on Empirical Risk Minimization with absolute loss \citep{Vap98}. For a Boolean target function, a hypothesis with $\ell_1$ error of $\eps$ also gives a hypothesis with classification error of $\eps$ (\eg~ \citep{KKMS05}). Therefore, as demonstrated by \citet{KKMS05}, $\ell_1$ linear regression is also useful for agnostic learning of Boolean functions.

\subsubsection{The Algorithm}
The basic idea of the algorithm is to find a small set $\S$ of indices for which there exists non-negative reals $\alpha_S$ with $\sum_{S \in \S} \alpha_S \leq1$ such that $\E[ |c(x) - \sum_{S \in \S} \alpha_S \cdot \ORR_S(x)|] \leq \epsilon$. Given $\S$, we can find a good non-negative linear combination as above using $\ell_1$ linear regression as described in Thm.~\ref{th:lae-lp} and Remark \ref{rem:non-neg-lae}. We will show that for any coverage function $c$, there exists a set of indices $\S$ (that depends on $c$) as above and moreover, that we can find such a set $\S$ using just random examples labeled by $c$. This will give us our proper learning algorithm.

To find a set $\S$ as above, we use the Fourier coefficients of the target coverage function. First, we will invoke Thm.~\ref{junta-coverage} to find a set $\tilde{I} \subseteq [n]$ of coordinates of size $O(1/\epsilon^2)$ just as in Thm.~\ref{thm:paccov} and show that we need only look at subsets of $\tilde{I}$ to find $\S$. Following this, we will prove that $\S$ can be identified using the following property: if $S \in \S$, then the Fourier coefficient $\hat{c}(S)$ is large enough in magnitude.

We begin with a simple lemma that shows that if a disjunction has a small number of variables and a significant coefficient in a coverage function then the corresponding (\ie with the same index set) Fourier coefficient is significant.
\begin{lemma}
Let $c = \sum_{S \subseteq [n]} \alpha_S \ORR_S$ be a coverage function. Then, for any $T \subseteq [n]$,  $|\hat{c}(T)| \geq 2^{-|T|} \cdot \alpha_T$.  \label{Fourier-finds-large}
\end{lemma}
\begin{proof}
From the proof of Lem.~\ref{lem:monotonicity}, we have for any non-empty $T \neq \emptyset$, ${\hat{c}(T) = - \sum_{ S \supseteq T} \alpha_S \cdot (\frac{1}{2^{|S|}})}$. Each term in the summation above is non-negative. Thus, $|\hat{c}(T)| \geq 2^{-|T|} \cdot \alpha_T$.
\end{proof}

We now show that for a coverage function $c \in \Cv$, using significant Fourier coefficients of $c$, we can identify a coverage function $c'$ of size at most $\min \{ \size(c), (1/\epsilon)^{O(\log{(1/\epsilon)})} \}$ that $\ell_1$ approximates $c$ within $\epsilon$. Moreover, $c'$ depends on just $O(1/\epsilon^2)$ variables.
\begin{lemma}
Let $c:\on^n \rightarrow [0,1]$ be a coverage function and $\epsilon > 0$. Let $I = \{i \in [n] \mid |\hat{c}(\{i\}) | \geq \epsilon^2/18\} $. Set $s_{\epsilon} = \min \{ \size(c),  |I|^{\log{(3/\epsilon)}}\}$ and let $$\T_\eps = \{ T \subseteq I \mid |\hat{c}(T)| \geq \epsilon^2/(9s_{\epsilon}) \text{ and } |T| \leq \log{(3/\epsilon)} \} \cup \{\emptyset\}.$$  Then,

\begin{enumerate}
	\item $|I| \leq 36/\epsilon^2$,
	\item There exists $c' \in \Cv$ such that $c' = \sum_{T \in \T_\eps} \alpha'_T \cdot \ORR_T$ and $\|c - c'\|_1 \leq \eps$.
\end{enumerate}
\label{large-Fourier-enough}
\end{lemma}
\begin{proof}
Let $\alpha_T$ for each $T \subseteq [n]$ be the coefficients of the disjunctions of $c$, that is, $c = \sum_{T \subseteq [n]} \alpha_T \cdot \ORR_T$ for non-negative $\alpha_T$ satisfying $\sum_{T \subseteq [n]} \alpha_T \leq 1$. Using Thm.~\ref{junta-coverage}, we know that the coverage function $c_I$ (defined as $c_I(x) = \E_{y \sim \on^{\bar{I}}}[  c(x_I \circ y  ]$) depends only on variables in $I$ and $\|c-c_I\|_1 \leq  \epsilon/3$. Therefore,  $c_I = \sum_{T \subseteq I} \beta_T \cdot \ORR_T$ for some constants $\beta_T \geq 0$, $\sum_{T \subseteq I} \beta_T \leq 1$.
Using Lem.~\ref{spectral-norm}, we obtain that $|I| \leq  36/\epsilon^2$ and thus $s_{\epsilon} \leq \min \{ \size(c), (6/\epsilon)^{2 \log{(3/\epsilon)}} \}$.

Let $\T_1= \{ T \subseteq I \mid 0 < \beta_T \leq \frac{\epsilon}{3s_\eps} \text{ and } |T| \leq \log{(\frac{3}{\epsilon})} \}$ and $\T_2 = \{ T \subseteq I \mid |T| > \log{(\frac{3}{\epsilon})}\}$. Consider any $T \subseteq I$ such that $T \notin T_1 \cup T_2$. Then $|T| < \log{(3/\epsilon)}$ and $\beta_T > \epsilon/(3s_\eps)$. This, by Lem.~\ref{Fourier-finds-large} applied to $c_I$,  implies that $|\hat{c}_I(T)| > \epsilon^2/(9s_\eps)$. By Lem.~\ref{Fourieravg}, $|\hat{c}(T)| = |\hat{c}_I(T)| \geq \epsilon^2/(9s_\eps)$. This implies that $T \in \T_\eps$ and thus every $T \subseteq I$ is in $\T_\eps \cup \T_1 \cup \T_2$.

Set $v = \sum_{T \in \T_2} \beta_T$ and let $c' = \sum_{T \in \T_\eps} \beta_T \ORR_T(x) + v$. Clearly $c' \in \Cv$ and $c' = \sum_{T \in \T_\eps} \alpha'_T \cdot \ORR_T$ for some coefficients $\alpha_T'$. For $c'$ we have:
\begin{align*}
\|c_I - c'\|_1 &=\E \left[ \left|\sum_{T \subseteq \I} \beta_T \cdot \ORR_T(x) - \left( \sum_{T \in \T_\eps} \beta_T \ORR_T(x) +  \sum_{ T \in \T_2} \beta_T \right) \right|\right]\\
&\leq \E \left[ \left|\sum_{T \in \T_1} \beta_T \cdot \ORR_T(x) \right|\right]+ \E \left[ \left|  \sum_{T \in \T_2} \beta_T \cdot (\ORR_T(x) -1) \right|\right]\\
&= \sum_{T \in \T_1} \E \left[ \left|\beta_T \ORR_T(x)\right| \right] + \sum_{T \in \T_2} \beta_T\cdot \pr [\ORR_T(x) = 0]
\end{align*}

By Thm.~\ref{junta-coverage}, $\size(c_I) \leq \size(c)$ and therefore $|\T_1| \leq s_\eps$. For each $T\in \T_2$, $|T| > \log{(\frac{3}{\epsilon})}$ which gives ${\pr [\ORR_T(x) = 0] < \epsilon/3}$. Therefore, $$\|c_I - c'\|_1  \leq |\T_1| \cdot \frac{\epsilon}{3s_\eps} + \epsilon/3 \cdot 1 \leq 2\epsilon/3.$$

Thus, $\| c - c' \|_1 \leq \|c-c_I\|_1+ \| c_I - c'\|_1 \leq \epsilon$.
\end{proof}
We can now describe and analyze our proper PAC learning algorithm for $\Cv$.
\label{app:proper-pac}
\begin{theorem}[Proper PAC Learning]
\label{thm:Proper-PAC learn-coverage}
 There exists an algorithm, that for any $\epsilon > 0$, given random and uniform examples of any $c \in \Cv$, with probability at least $2/3$, outputs $h \in \Cv$ such that $\|h-c\|_1 \leq \eps$. Further, $\size(h)=O(s_\eps/\epsilon^2)$ and the algorithm runs in time $\tilde{O}(n) \cdot \poly(s_\eps/\epsilon)$ and uses $\log{(n)} \cdot \tilde{O}(s_\eps^2/\epsilon^4)$ random examples, where $s_\eps = \min\{\size(c), (12/\eps)^{2\lceil \log{(6/\epsilon)}\rceil}\}$.
\end{theorem}

\begin{algorithm}
\caption{\textit{Proper} PAC Learning of Coverage Functions}
\begin{algorithmic}[1]
\STATE Set $\theta = \frac{\epsilon^2}{108}$, $s_\eps = \min\{ \size(c), (12/\epsilon)^{\lceil \log{(6/\epsilon)}\rceil} \}$.
\STATE Draw a random sample of size $m_1=O(\log{(n)}/\eps^4)$ and use it to
estimate, $\hat{c}(\{i\})$ for all $i$.
\STATE Set $\tilde{\I} = \{ i \in [n] \mid |\tilde{c}(\{i\})| \geq \theta\}$.
\STATE $\S \leftarrow \{\emptyset\}$.
\STATE Draw random sample $R$ of size $m_2 = O(s^2_\eps\log{(s_\eps/\eps)}/\eps^4)$.
\FOR{$t=1$ to $\lceil \log(6/\epsilon)\rceil$}
\FOR{each set $T \in \S$ of size $t-1$ and $i\in \tilde{\I}\setminus T$ }
\STATE Use $R$ to estimate the coefficient $\hat{c}({T \cup \{i\}})$.
\STATE If $|\tilde{c}({T \cup i}) | \geq \theta$ then $\S \leftarrow \S \cup \{T \cup \{i\}\}$
\ENDFOR
\ENDFOR
\STATE Draw a random sample $R$ of size $m_3=O(s_\eps/\eps^4)$ and use $\ell_1$ linear regression to minimize $\sum_{(x,y) \in R}[ |y - \sum_{S \in \S} \alpha_S \cdot \ORR_S(x)|]$ subject to $\sum_{S \in \S} \alpha_S \leq 1$ and $\alpha_S \geq 0$ for all $S \in \S$. Let $\alpha^*_S$ for each $S \in \S$ be the solution.
\RETURN $\sum_{S \in \S}\alpha^*_S\cdot \ORR_S$.
\end{algorithmic}
\label{Proper-PAC learn-coverage}
\end{algorithm}

\begin{proof}
We first describe the algorithm and then present the analysis.  We break the description of the algorithm into three stages. The first two stages are similar to those of Algorithm \ref{PAC learn-coverage}.

Let $\theta = \eps^2/108$. In the first stage, our algorithm finds a set $\tilde{I}$ of variables that contains $\I = \{ i \in [n] \mid |\hat{c}(\{i\})| \geq (\epsilon/2)^2/18\}$. We do this by estimating all the singleton Fourier coefficients, $\{ \hat{c}(\{i\}) \mid i \in [n] \}$ within $\theta/2$ with (overall) probability at least $8/9$. As before,  we denote the estimate of $\hat{c}(S)$ by $\tilde{c}(S)$ for any $S$. We then set $\tilde{\I} =\{  i \in [n] \cond \tilde{c}(\{i\})| \geq \theta\}$.
If all the estimates are within $\theta/2$ of the corresponding coefficients then for every $i \in \I$, $\tilde{c}(\{i\}) \geq \eps^2/72 - \theta/2 = \eps^2/108=\theta$. Therefore $i\in \tilde{\I}$ and hence $\I \subseteq \tilde{\I}$.

In the second stage, the algorithm finds a set $\S \subseteq 2^{\tilde{\I}}$ such that the set of all large Fourier coefficients $\T$ is included in $\S$. Just as in Thm.~\ref{thm:paccov}, this is done iteratively starting with $\S = \{\emptyset\}$. In every iteration, for every set $S$ that was added in the previous iteration and every $i \in \tilde{\I}\setminus S$, it estimates $\hat{c}(S \cup \{i\})$ within $\epsilon^2/(108s_\eps)$ (the success probability for all estimates in this phase will be $8/9$). If $|\tilde{c}(S \cup \{i\})| \geq \epsilon^2/(54s_\eps)$ then $S \cup \{i\}$ is added to $\S$. The iterative process is run for at most $\lceil \log{(6/\epsilon)} \rceil$ iterations.

Finally, the algorithm draws a random sample $R$ of size $m_3$ and uses $\ell_1$ linear regression (Thm.~\ref{th:lae-lp}) to minimize $\sum_{(x,y) \in R}[ |y - \sum_{S \in \S} \alpha_S \cdot \ORR_S(x)|]$ subject to $\sum_{S \in \S} \alpha_S \leq 1$ and $\alpha_S \geq 0$ for all $S \in \S$. Let $\alpha^*_S$ for each $S \in \S$ be the solution. Here $m_3$ is chosen so that, with probability at least $8/9$, $\E [|c(x) - \sum_{S \in \S} \alpha^*_S \cdot \ORR_S(x)]$ (the true error of $\sum_{S \in \S} \alpha^*_S \cdot \ORR_S(x)$) is within $\eps/2$ of the optimum. Standard uniform convergence bounds imply that $m_3=O(|\S|/\epsilon^2)$ examples suffice. The algorithm returns $\sum_{S \in \S} \alpha^*_S \cdot \ORR_S$ as the hypothesis.

We can now prove the correctness of the algorithm assuming that all the estimates are successful. By an argument similar to the one presented in the proof of Thm.~\ref{thm:paccov}, we can verify that $I \subseteq \tilde{I}$ and that $\T \subseteq \S$. Using Thm.~\ref{large-Fourier-enough} and the facts that $\S \supseteq \T$ and $\tilde{I} \supseteq I$, we obtain that there must exist non-negative $\alpha'_S$ with  $\sum_{S \in \S } \alpha'_S \leq 1$ such that $\E[|c(x) - \sum_{S \in \S } \alpha_S \cdot \ORR_S(x)|] \leq \epsilon/2$. Thus $\ell_1$ linear regression in the third stage will return coefficients $\alpha^*_S$ for each $S \in \S$ such that $\E [|c(x) - \sum_{S \in \S} \alpha^*_S \cdot \ORR_S(x)|] \leq \eps$.

We now analyze the running time and sample complexity of the algorithm. The analysis for the first two stages is similar to the one presented in Thm.~\ref{thm:paccov}.
\begin{itemize}
	\item Just as in the proof of Thm.~\ref{thm:paccov}, all the singleton coefficients can be estimated in time $O(n \log{(n)} /\epsilon^4)$ and samples $O(\log{(n)}/\epsilon^4)$ with confidence at least $8/9$.
    \item For every $i$ such that $\hat{c}(\{i\})$ was estimated within $\theta/2$ and $|\tilde{c}(\{i\})| \geq \theta$, we have that $|\hat{c}(\{i\})| \geq \theta/2$. This implies that $|\tilde{\I}| \leq 2/(\theta/2) = O(1/\eps^2)$. Similarly, for every $S \subseteq \tilde{I}$ such that $\hat{c}(S)$ was estimated within $\epsilon^2/(108s_\eps)$ and $|\tilde{c}(S)| \geq \epsilon^2/(54s_\eps)$, we have that $|\hat{c}(S)| \geq \epsilon^2/108s_\eps$. Thus using Lem.~\ref{spectral-norm}, $|\S| = O(s_\eps/\eps^2)$.
	\item In the second stage the algorithm only estimates coefficients with indices that are subsets of $$\S' = \{ S \cup \{i\} \cond |\tilde{c}(S)| \geq \epsilon^2/(54s_\eps) \mbox{ and } i \in \tilde{\I} \mbox{ and } |S| \leq \log{(6/\epsilon)} \}. $$ We can conclude that $|\S'| = O(s_\eps/\eps^4)$ and, as in the proof of Thm.~\ref{thm:paccov}, the estimation succeeds with probability at least $8/9$ using $\tilde{O}(s_\eps^2/\epsilon^4)$ examples and running in time $\tilde{O}(s_\eps^2/\epsilon^8)$. 
\item Finally, using Thm.~\ref{th:lae-lp}, $\ell_1$ linear regression will require $O(|\S|/\epsilon^2) = O(s_\eps/\epsilon^4)$ random examples and runs in time polynomial in $|S|$ and $1/\eps$, that is $\poly(s_\eps/\eps)$.
\end{itemize}

Overall, the algorithm succeeds with probability at least $2/3$, runs in time $\tilde{O}(n) \cdot \poly(s_\eps/\epsilon)$ and uses $\log{(n)} \cdot \tilde{O}(s_\eps^2/\epsilon^4)$ random examples.
\end{proof}

\newcommand{\q}{\mathbf{q}}
\section{Agnostic Learning on Product and Symmetric Distributions} \label{sec:agnosticuniform}
In this section, we give optimal algorithms for agnostically learning (we give a formal definition below)
coverage functions on arbitrary product and symmetric distributions. Recall that a distribution is symmetric if the associated probability density function is symmetric on $\on^n$. We begin by recalling the definition of agnostic learning with $\ell_1$-error.
\begin{definition}
\label{def:agnostic}
Let $\F$ be a class of real-valued functions on $\on^n$ with range in $[0,1]$ and let $\D$ be any fixed distribution on $\on^n$. For any distribution $\P$ over $\on^n \times [0,1]$, let $\mbox{opt}(\P,\F)$ be defined as: $\mbox{opt}(\P,\F) =  \inf_{f \in \F} \E_{(x,y) \sim \P} [ |y - f(x) |] .$ An algorithm $\A$, is said to agnostically learn $\F$ on $\D$ if for every \em{excess error} $\epsilon> 0$ and any distribution $\P$ on $\on^n \times [0,1]$ such that the marginal of $\P$ on $\on^n$ is $\D$, given access to random independent examples drawn from $\P$, with probability at least $\frac{2}{3}$, $\A$ outputs a hypothesis $h$ such that $\E_{(x,y) \sim \P} [ |h(x)- y| ] \leq \mbox{opt}(\P, \F) + \epsilon.$
\end{definition}

Our learning result is based on a simple observation (a special case of which is implicit in \citep{CKKL12}) that an $\ell_1$-approximation on a distribution $\D$ for all monotone disjunctions by linear combination of functions from a fixed set of functions yields a similar approximation for $\Cv$ on $\D$. For completeness a proof is included in App.~\ref{app:proofs}.
\begin{lemma}\label{lem:disj2cov}
Fix a distribution $\D$ on $\on^n$. Suppose there exist functions $\phi_1, \phi_2, \ldots, \phi_m : \zo^n \rightarrow \R$ such that for any $S \subseteq [n]$, there are  reals $\beta^S_1, \beta^S_2, \ldots, \beta^S_m$ such that $\|\ORR_S -\sum_{j = 1}^m \beta^S_j \cdot \phi_j\|_1 = \E_{x \sim \D}[ |\ORR_S(x) -\sum_{j = 1}^m \beta^S_j \cdot \phi_j(x)|]  \leq \epsilon.$ Then, for every coverage function $c \in \Cv$, there exist reals $\beta_1 , \beta_2 , \ldots, \beta_m$ such that $\|c -\sum_{j = 1}^m \beta_j \cdot \phi_j\|_1 = \E_{x \sim \D}[ |c(x) - \sum_{j = 1}^m \beta_j  \cdot \phi_j(x)|]  \leq \epsilon.$
\end{lemma}

A natural and commonly used set of basis functions is the set of all monomials on $\on^n$ of some bounded degree.  It is easy to see that on product distributions with constant bias, disjunctions longer than some constant multiple of $\log(1/\eps)$ are $\eps$-close to constant 1. Therefore degree $O(\log(1/\eps))$ suffices for $\ell_1$ approximation on such distributions. This simple argument does not work for general product distributions. However it was shown by \citet{BOW08} that the same degree (up to a constant factor) still suffices in this case. Their argument is based on the analysis of noise sensitivity under product distributions and implies additional interesting results. A simpler proof of this fact also appears in \citep{FeldmanKothari:14symm}, who also show that the same holds if the distribution is uniform over points of Hamming weight $k$, for any fixed $k \in \{0,\ldots,n\}$.
\begin{lemma}[\citealp{FeldmanKothari:14symm}]
For $0 \leq k \leq n$, let $\Pi_k$ denote the uniform distribution over points of Hamming weight $k$. For every disjunction $f$ and $\eps >0$, there exists a polynomial $p$ of degree at most $O(\log{(1/\epsilon)})$ such that $\E_{x\sim \Pi_k}[|f(x) - p(x)] \leq \eps$.
 \label{monotone-sym}
\end{lemma}

This result implies a basis for approximating disjunctions over arbitrary symmetric distributions. All we need is to partition the domain $\on^n$ into $\cup_{0 \leq k \leq n}S_k$ layers and use a (different) polynomial for each layer. Formally, the basis now contains functions of the form $\mathrm{IND}(k) \cdot \chi$, where $\mathrm{IND}$ is the indicator function of being in layer of Hamming weight $k$ and $\chi$ is a monomial of degree $O(\log(1/\eps))$. In contrast, as shown in \citep{FeldmanKothari:14symm}, there exist symmetric distributions over which disjunctions cannot be $\ell_1$-approximated by low-degree polynomials.

These results together with the $\ell_1$ regression in Thm.~\ref{th:lae-lp} immediately yield an agnostic learning algorithm for the class of coverage functions $\Cv$ over any product or symmetric distribution.
\begin{theorem}[Thm.~\ref{th:introsym} restated]
There exists an algorithm that for any product or symmetric distribution $\D$ agnostically learns $\Cv$ with excess $\ell_1$ error $\eps$ in time $n^{O(\log{1/\epsilon})}$. \label{th:agnosticsym}
\end{theorem}

We now remark that any algorithm that agnostically learns the class of coverage functions on $n$ inputs on the uniform distribution on $\zo^n$ in time $n^{o(\log{(\frac{1}{\epsilon})})}$ would yield a faster algorithm for the notoriously hard problem of learning sparse parities with noise. The reduction only uses the fact that coverage functions include all monotone disjunctions and follows from the results in \citep{KKMS05,Feldman:12jcss} (see \citep{FeldmanKothari:14symm} for details).

\subsection{Proper Agnostic Learning}
For the special case of \emph{bounded} product distributions, that is, product distributions with one dimensional marginal expectations bounded away from $0$ and $1$ by constants, we can in fact obtain a \emph{proper} agnostic learner for $\Cv$. The proof is based on approximating coverage functions by truncating the expansion of a coverage function in terms of monotone disjunctions (Lem.~\ref{disj-rep}) to keep only the terms corresponding to short disjunctions. We show that such a truncation is enough to approximate the function with respect to $\ell_1$ error. Recall that by Lem.~\ref{disj-rep}, such a truncation is itself a coverage function.

\begin{lemma}
Let $c(x) = \sum_{S \subseteq [n]} \alpha_S \cdot \ORR_S(x)$ be a coverage function with range in $[0,1]$ and $\epsilon > 0$. Let $\D$ be any product distribution such that the one-dimensional marginal expectations of $\D$ are at most $1-\kappa$.  Then, for $k = 2/\kappa \cdot \lceil \log{(\frac{1}{\epsilon})} \rceil $ the coverage function $c' = \sum_{|S| > k} \alpha_S + \sum_{|S| \leq k} \alpha_S \cdot \ORR_S$ satisfies  $\E_{x \sim \D}[| c- c'|] \leq \epsilon $. \label{disj-approx}
\end{lemma}

\begin{proof}
By Lem.~\ref{disj-rep} and the fact that $c$ is a coverage function we obtain that $c'$ is a non-negative combination of monotone disjunctions with the sum of coefficients being at most 1, that is a coverage function itself.
For each $x \in \on^n$, non-negativity of the coefficients $\alpha_S$ implies $c(x) \leq c'(x)$.

For a monotone disjunction $\ORR_S$, observe that if $|S| > k$, then $\pr_{x \sim \D}[ \ORR_S(x) = 0]\leq (1-\kappa)^k \leq \epsilon$. Thus:
Now,
\begin{align*}
\|c- c'\|_1 &= \E [ |c(x) - c'(x)| ]
= \E_{x \sim \D} [ c'(x) - c(x)]\\& = \sum_{S \subseteq [n], |S| > k} \E_{x \in \D} [ \alpha_S - \alpha_S \cdot \ORR_S ] = \sum_{S \subseteq [n], |S| > k} \pr_\D[\ORR_S = 0]\cdot \alpha_S\\ &\leq \sum_{S \subseteq [n], |S| > k} \alpha_S\cdot \epsilon
\leq \epsilon .
\end{align*}
\end{proof}

As an immediate corollary of Lem.~\ref{disj-approx} and Thm.~\ref{th:lae-lp}, we obtain an algorithm for agnostic learning of coverage functions on the uniform distribution.

\begin{theorem}
There exists an algorithm, that agnostically learns the class $\Cv$ on any bounded product distributions in time $n^{O(\log{(\frac{1}{\epsilon})})}$.  Further, the hypothesis returned by the algorithm is itself a coverage function. \label{agnostic-coverage}
\end{theorem}
\begin{proof}
Lem.~\ref{disj-approx} shows that every coverage function can be approximated by a non-negative linear combination of monotone disjunctions of length $\log{(\frac{1}{\epsilon})}$ within $\epsilon$ in the $\ell_1$-norm. Now, Thm.~\ref{th:lae-lp} with Remark \ref{rem:non-neg-lae} immediately yields an agnostic learning algorithm.
\end{proof}

\section{Privately Releasing Monotone Conjunction Counting Queries}
\label{sec:pdr}
\label{sec:privacy}
In this section we use our learning algorithms to derive privacy-preserving algorithms for releasing monotone conjunction (equivalently, disjunction) counting queries. We begin with the necessary formal definitions.
\subsection{Preliminaries}
\textbf{Differential Privacy:}
We use the standard formal notion of privacy, referred to as \emph{differential privacy}, proposed by  \citet{DMNS06}.
For some domain $X$, we will call $D\subseteq X $ a \emph{data set}. data sets $D,D'\subset X$ are \emph{adjacent} if one can be obtained from the other by adding a single element. In this paper, we will focus on Boolean data sets, thus, $X \subseteq \on^n$ for $n \in \N$. We now define a differentially private algorithm. In the following, $A$ is an algorithm that takes as input a data set $D$ and outputs an element of some set $R$.

\begin{definition}[Differential privacy \citep{DMNS06}]
An (randomized) algorithm $A:2^X \rightarrow R$ is \emph{$\eps$-differentially private} if for all $r \in R$ and every pair of adjacent data sets $D,D'$, we have $\pr[A(D) = r] \leq e^\eps\pr[A(D') = r]$. \label{def:privacy}
\end{definition}

\textbf{Private Counting Query Release:} We are interested in algorithms that answer predicate \emph{counting} queries on Boolean data sets. A predicate counting query finds the fraction of elements in a given data set that satisfy the predicate. More generally, given a query  $c:\on^n \rightarrow [0,1]$, a counting query corresponding to $c$ on a data set $D \subseteq \on^n$ of size $m := |D|$, expects in reply $q_c(D) = 1/m\sum_{r \in D} c(r)$. In our applications, we will only insist on answering the counting queries approximately, that is, for some $\tau > 0$, an approximate counting query in the setting above expects a value $v$ that satisfies $|v - 1/m\sum_{r \in D} c(r) | \leq \tau$. We refer to $\tau$ as the {\em tolerance} of the counting query. A class of queries $\C$ mapping $\on^n$ into $[0,1]$, thus induces a \emph{counting query function} $\CQ_D: \C \rightarrow [0,1]$ given by $\CQ_D(c) = q_c(D)$ for every $c \in \C$. For a class $\C$ of such functions and a data set $D$, the goal of a data release algorithm is to output a summary $H: \C \rightarrow [0,1]$ that provides (approximate) answers  to queries in $\C$. A \emph{private} data release algorithm additionally requires that $H$ be produced in a differentially private way with respect to the participants in the data set. One very useful way of publishing a summary is to output a \emph{synthetic data set} $\hat{D} \subseteq \on^n$ such that for any query $c \in \C$, $q_c(\hat{D})$ is a good approximation for $q_c(D)$. Synthetic data sets are an attractive method for publishing private summaries as they can be directly used in software applications that are designed to run on Boolean data sets in addition to being easily understood by humans.

For a class $\C$ of queries mapping $\on^n$ into $[0,1]$,  and a distribution $\Pi$ on $\C$, an algorithm $A$ $(\alpha,\beta)$-answers queries from $\C$ over a data set $D$ on the distribution $\Pi$, if for $H=A(D)$, $\pr_{f \sim \Pi} [|\CQ_D(f)-H(f)| \leq \alpha] \geq 1-\beta.$ For convenience we will only measure the average error $\bar{\alpha}$ and require that $\E_{f \sim \Pi} [|\CQ_D(f)-H(f)|] \leq \bar{\alpha}$. Clearly, one can obtain an $(\alpha, \beta)$-query release algorithm from an $\bar{\alpha}$-average error query release algorithm by setting $\bar{\alpha} = \alpha \cdot \beta$.

The key observation for obtaining conjunction query release algorithms from learning algorithms for coverage functions is that for any data set $D$ and the query class of monotone conjunctions, $1-\CQ_D$ is a coverage function. Namely, for any $S \subseteq [n]$, let $\ANDD_S$ be the monotone conjunction $\wedge_{i \in S} x_i$ which equals $1$ iff each $x_i = -1$ for $i \in S$ and for $x \in \on^n$ let $S_x \subseteq [n]$ be the set such that $x_i = -1$ iff $i \in S_x$. Then $c_D(x) \doteq 1-\CQ_D(\ANDD_{S_x})$ is a coverage function. We include a simple proof of this fact for completeness.
\begin{lemma}
\label{lem:release-is-coverage}
For a data set $D$, let $c_D:\on^n \rightarrow [0,1]$ be defined as $c_D(x) = 1-\CQ_D(\ANDD_{S_x})$. Then $c_D$ is a coverage function.
\end{lemma}
\begin{proof}
Let $x \in \on^n$. By definition, $$c_D(x) = 1-\CQ_D(\ANDD_{S_x}) =1-\frac{1}{|D|} \sum_{z \in D} \ANDD_{S_x} (z) = \frac{1}{|D|} \sum_{z \in D} (1-\ANDD_{S_{x}}(z)).$$ Note that $$1-\ANDD_{S_x}(z) =  1- \bigwedge_{x_i = -1} [z_i = -1] = \bigvee_{z_i = 1} [x_i = -1] = \ORR_{S_{-z}}(x).$$
Then, $$c_D(x) =  \sum_{z \in D} \frac{1}{|D|} \cdot \ORR_{S_{-z}}(x) .$$
\end{proof}


Lem.~\ref{lem:release-is-coverage} implies that for the class of monotone conjunctions $\C$, the set of functions $\{1-\CQ_D | D\subseteq \on^n\}$ is a subset of $\Cv$.  Additive error approximation for $c_D = 1-\CQ_D$ is equivalent to additive error approximation of $\CQ_D$. Therefore to obtain a private release algorithm with average error $\bar{\alpha}$ relative to a distribution $\Pi$ over monotone conjunctions, it is sufficient to produce a hypothesis $h$ that satisfies, $\E_{x \sim \Pi}[|c_D(x)-h(x)|] \leq \bar{\alpha}$, where we view $\Pi$ also as a distribution over vectors corresponding to conjunctions (with $x$ corresponding to $\ANDD_{S_x}$).
Note that the average error is exactly the $\ell_1$-error in approximation of $\CQ_D$ over distribution $\Pi$. A monotone conjunction query of length $k$ corresponds to a point in $\on$ that has exactly $k$ $(-1)$s.

To convert our learning algorithms to differentially-private release algorithms we rely on the following proposition that \citet{GHRU11} prove using technique from \citep{BDMN05}.

\begin{proposition}[\citealp{GHRU11}]\label{prop:ghru}
Let $\A$ denote an algorithm that uses $q$ counting queries of tolerance $\tau$ in its computation. Then for every $\eps,\delta > 0$, with probability $1-\delta$, $\A$ can be simulated in an $\eps$-differentially private way provided that the size of data set $|D|\geq q(\log q + \log(1/\delta))/(\eps \cdot \tau)$. Simulation of each query of $\A$ takes time $O(|D|)$.
\end{proposition}

\subsection{Releasing $k$-way Marginals with Low Average Error}
We now describe a differentially private algorithm for releasing monotone conjunction counting queries of length $k$ with low average error. The result is based on a simple implementation of the $\ell_1$ linear regression algorithm using tolerant counting query access to the data set $D$. Let $\C_k$ be the class of all monotone conjunctions of length $k \in [n]$ and let $\Pi_k$ denote the uniform distribution over $\C_k$.
\begin{theorem} \label{th:shortrelease}
 For every $\epsilon > 0$, there is an $\epsilon$-differentially private algorithm which for any data set $D \subseteq \on^n$ of size $n^{\Omega(\log{(1/\bar{\alpha})})} \cdot \log{1/\delta} /\epsilon$, with probability at least $1-\delta$ publishes a data structure $H$ that answers counting queries for $\C_k$ with an average error of at most $\bar{\alpha}$ relative to $\Pi_k$. The algorithm runs in time $n^{O(\log{(1/\bar{\alpha})})} \cdot \log{(1/\delta)}/\epsilon$ and the size of $H$ is $n^{O(\log{(1/\bar{\alpha})})}$.
\end{theorem}
\begin{proof}

In the light of the discussion above, we show how to implement the algorithm described in Thm. \ref{th:introsym} to learn $\{c_D \mid D \subseteq \on^n\}$ with tolerant value query access to the data set $D$.  We will simulate the algorithm described in Thm. \ref{th:introsym} over distribution $\Pi_k$ and with excess $\ell_1$-error of $\bar{\alpha}/2$. The algorithm uses $\ell_1$ linear regression to find a linear combination of $t = n^{O(\log{(1/\bar{\alpha})})}$ monomials that best fits random examples $(x^i, y^i)$.

We can simulate random examples of $c_D$ by drawing $x$ from $\Pi_k$ and making the counting query on the conjunction $\ANDD_{S_{x}}$. Since we can only use $\tau$-tolerant queries, we are guaranteed that the value we obtain, denote it by $\tilde{c}_D(x)$, satisfies $|c_D(x) - \tilde{c}_D(x)| \leq \tau$. This additional error in values has average value of at most $\tau$ and hence can cause $\ell_1$ linear regression to find a solution whose average absolute error is up to $2\tau$ worse than the average absolute error of the optimal solution. This means that we are guaranteed that the returned polynomial $H$ satisfies $\E[|c_D(x) - H(x)|] \leq 2\tau + \bar{\alpha}/2$.
We set $\tau = \bar{\alpha}/4$ and obtain that the error is at most $\bar{\alpha}$.
This implementation makes $n^{O(\log{(1/\bar{\alpha})})}$ $(\bar{\alpha}/4)$-tolerant counting queries to the data set $D$ and uses $n^{O(\log{(1/\bar{\alpha})})}$ time to output $H$. Applying Proposition \ref{prop:ghru}, we obtain that there exists a $\epsilon$-differentially private algorithm to compute an $H$ as above with the claimed bounds on the size of $D$ and running time.

\end{proof}

\subsection{Releasing All Marginals with Low Average Error}
Next, we show that we can implement the algorithm from Thm. \ref{thm:paccov} using tolerant counting query access to the data set $D$ and thereby obtain a private data release algorithm for monotone conjunctions with low average error relative to the uniform distribution. Notice that since we only promise low-average error over all monotone conjunction counting queries, for some $k$'s the average error on conjunctions of length $k$ can be very large. 


\begin{theorem}[Thm.~\ref{thm:nosynth-intro} restated]
Let $\C$ be the class of all monotone conjunctions. For every $\eps,\delta > 0$, there exists an $\epsilon$-differentially private algorithm which for any data set $D \subseteq \on^n$ of size $\tilde{\Omega}(n \log(1/\delta)/(\eps\bar{\alpha}^6))$, with probability at least $1-\delta$, publishes a data structure $H$ that answers counting queries for $\C$ with respect to the uniform distribution with average error of at most $\bar{\alpha}$. The algorithm runs in time $\tilde{O}(n^2 \log(1/\delta)/(\eps\bar{\alpha}^{10}))$ and the size of $H$ is $\log n \cdot \tilde{O}(1/\bar{\alpha}^4)$.
\label{thm:nosynth}
\end{theorem}
\begin{proof}
In the algorithm from Thm.~\ref{thm:paccov} the random examples of the target coverage function $c$ are used only to estimate Fourier coefficients of $c$ within tolerance $\theta/2$. Thus to implement the algorithm from Thm.~\ref{thm:paccov}, it is sufficient to show that for any index set $T \subseteq [n]$, we can compute $\widehat{c}_D(T)$ within $\theta/2$ using tolerant counting query access to $D$.

Consider any set $T \subseteq [n]$ and recall that for any $x \in \on^n$, $S_x = \{ i \mid x_i = -1\}$. From the proof of Lemma \ref{lem:release-is-coverage}, we have: $c_D = 1/|D| \sum_{z \in D} \ORR_{S_{-z}}.$ Then, we have
\begin{align}
\widehat{c_D}(T) &= \E_{x \sim \U_n} [c_D(x) \cdot \chi_T(x)] = \E_{x \sim \U_n}  \left[\frac{1}{|D|}\sum_{z \in D}  \ORR_{S_{-z}}(x) \cdot \chi_T(x) \right] \nonumber \\
&=   \frac{1}{|D|} \E_{x \sim \U_n}  [\ORR_{S_{-z}}(x) \cdot \chi_T(x)] = \frac{1}{|D|}  \sum_{z \in D} \widehat{\ORR}_{S_{-z}}(T) . \label{countingquery}
\end{align}

Define $F_T(z) = (1+\widehat{\ORR}_{S_{-z}}(T))/2$. Now $F_T$ is a function with range $[0,1]$ and from equation \eqref{countingquery} above,  we observe that $\widehat{c_D}(T)$ can be estimated with tolerance $\theta/2$ by making a counting query for $F_T$ on $D$ with tolerance $\theta/4$. We now note that $\ell_1$-error of hypothesis $h$ over the uniform distribution on $\on^n$ is the same as the average error $\bar{\alpha}$ of answering counting queries using $h$ over the uniform distribution on monotone disjunctions. Therefore $\theta/4 = (\bar{\alpha})^2/24$. The number of queries made by the algorithm is exactly equal to the number of Fourier coefficients estimated by it which is $O(n + 1/\bar{\alpha}^4)$. The output of the PAC learning algorithm is a linear combination of $O(1/\bar{\alpha}^2)$ parities over a subset of $O(1/\bar{\alpha}^2)$ variables and hence requires $\log{n} \cdot \tilde{O}(1/\bar{\alpha}^4)$ space. Note that given correct estimates of Fourier coefficients, the PAC learning algorithm is always successful. By applying Proposition \ref{prop:ghru}, we can obtain an $\eps$-differentially private execution of the PAC learning algorithm that succeeds with probability at least $1-\delta$ provided that the data set size is $$\Omega\left(\frac{(n + \bar{\alpha}^4) \log(n + \bar{\alpha}^4) +\log(1/\delta)}{\eps \bar{\alpha}^2}\right) = \tilde{\Omega}(n \log(1/\delta)/(\eps\bar{\alpha}^6)).$$
The running time is dominated by the estimation of Fourier coefficients and hence is $O(n + 1/\bar{\alpha}^4) = O(n/\bar{\alpha}^4)$ times the size of the data set.
\end{proof}

We now use our proper PAC learning algorithm for coverage functions to obtain an algorithm for synthetic data set release for answering monotone conjunction counting queries.
\begin{theorem}
\label{thm:synth}
Let $\C$ be the class of all monotone conjunctions. For every $\eps,\delta > 0$, there exists an $\epsilon$-differentially private algorithm which for any data set $D \subseteq \on^n$ of size $n \cdot \bar{\alpha}^{-\Omega(\log{(1/\bar{\alpha})})} \cdot \log(1/\delta)/\epsilon$, with probability at least $1-\delta$, releases a synthetic data set $\hat{D}$ of size $\bar{\alpha}^{-O(\log{(1/\bar{\alpha})})}$ that can answer counting queries for $\C$ with respect to the uniform distribution with  average error of at most $\bar{\alpha}$. The algorithm runs in time $n^2 \cdot \bar{\alpha}^{-O(\log{(1/\bar{\alpha})})} \cdot \log(1/\delta)/\epsilon$.

\end{theorem}
\begin{proof}
We will show that we can implement our proper PAC learning algorithm for coverage functions (Algorithm \ref{Proper-PAC learn-coverage}) to learn $\{c_D \mid D \subseteq \on^n\} \subseteq \Cv$ using tolerant counting query access to the data set $D$. Algorithm \ref{Proper-PAC learn-coverage} returns a hypothesis $H:\on^n \rightarrow [0,1]$ given by a non-negative linear combination of monotone disjunctions that $\ell_1$-approximates $c_D$. Then, $H'$ defined by $H'(x) = 1-H(x)$, $\ell_1$ approximates $\CQ_D$. We will then show that we can construct a data set $\hat{D}$ using $H$, such that $c_{\hat{D}}$ computes the same function as $H$, up to a small discretization error. Now,  $\CQ_{\hat{D}}(x)= 1-c_{\hat{D}}(x)$, and thus, answering monotone conjunction queries based on the data set $\hat{D}$ incurs low average error. As before, an $\eps$-differentially private version of this algorithm is then obtained by invoking Proposition \ref{prop:ghru}.

Algorithm \ref{Proper-PAC learn-coverage} uses the random examples from the target function to first estimate certain Fourier coefficients and then to run the $\ell_1$ linear regression (Thm.~\ref{th:lae-lp}) to find the coefficients of the linear combination or disjunctions. We have already shown that tolerant counting queries on the data set can be used to estimate Fourier coefficients.

Recall that we proved that there exist monotone disjunctions $\ORR_{S_1}, \ORR_{S_2},\ldots, \ORR_{S_t}$ and non-negative reals $\gamma_i$ for $ i \in [t]$ satisfying $\sum_{i \in [t]} \gamma_i \leq 1$ such that $\E[|c_D(x) - \sum_{i \in [t]} \gamma_i \cdot \ORR_{S_i}(x)|] \leq \Delta$, for some $t$ and $\Delta$. The algorithm, at this stage, runs $\ell_1$ linear regression to compute non-negative coefficients for appropriately chosen monotone disjunctions. As we showed in Thm.~\ref{th:shortrelease}, $\ell_1$ linear regression can be implemented via $\tau$-tolerant counting queries and results in the additional error of $2\tau$. That is we obtain non-negative reals $\gamma^*_i$ for $i \in [t]$ such that $\E[|c_D(x) - \sum_{i \in [t]} \gamma^*_i \cdot \ORR_{S_i}(x)|] \leq \Delta + \eta + 2\tau$.
%
%

Thus, we simulate Algorithm \ref{Proper-PAC learn-coverage} with excess $\ell_1$-error of $\bar{\alpha}/2$ and let $\tau$ be the minimum of tolerance required for estimating Fourier coefficients and $\bar{\alpha}/8$. The error of the coverage hypothesis function $H$ is then at most $\bar{\alpha}/2 + 2\tau \leq 3\bar{\alpha}/4$. Inspecting the proof of Thm.~\ref{thm:Proper-PAC learn-coverage} shows that $\tau = \bar{\alpha}^{O(\log{(1/\alpha)})}$ suffices. Thus, to summarize, we obtain that there exists an algorithm that makes $n+ \bar{\alpha}^{-O(\log{1/\bar{\alpha}})}$
$\tau$-tolerant counting queries to the data set $D$ (for $\tau$ as above) and uses $n \cdot \bar{\alpha}^{-O(\log{1/\bar{\alpha}})}$ time to output a non-negative linear combination $H$ of at most $\bar{\alpha}^{-O(\log{1/\bar{\alpha}})}$ monotone disjunctions that satisfies $\|c_D - H\|_1 \leq 3\bar{\alpha}/4$. Applying Proposition \ref{prop:ghru}, we obtain that there exists a $\epsilon$-differentially private algorithm to compute an $H$ as above with the claimed bounds on the size of $D$ and running time.

We now convert our hypothesis $H(x) = \sum_{i \in [t]} \gamma^*_i \cdot \ORR_{S_i}(x)$ into a data set by using the converse of Lem.~\ref{lem:release-is-coverage}. For each $\ORR_S$, let $x^S \in \on^n$ be defined by: for all $j\in [n]$, $x_j^S = -1$ if and only if $j\in S$. Our goal is to construct a data set $\hat{D}$ in which each $x^{S_i}$ for $i\in [t]$ has a number of copies that is proportional to $\gamma^*_i$ since this would imply that $c_{\hat{D}} = H$. To achieve this we first round down each $\gamma^*_i$ to the nearest multiple of $\bar{\alpha}/(4t)$ and let $\tilde{\gamma}_i$ denote the result. The function $\tilde{H}(x) = \sum_{i \leq t} \tilde{\gamma}_i \cdot \ORR_{S_i}(x)$ satisfies $\|\tilde{H} - H\|_1 \leq \bar{\alpha}/4$ and hence $\|\tilde{H} - c_D\|_1 \leq \bar{\alpha}$. Now we let $\hat{D}$ be the data set in which each $x^{S_i}$ for $i\in [t]$ has $4t\tilde{\gamma}_i/\bar{\alpha}$ copies (an non-negative integer by our discretization). From Lem.~\ref{lem:release-is-coverage} we see that $c_{\hat{D}} = \tilde{H}$. Note that the size of $\hat{D}$ is at most $4t/\bar{\alpha} = \bar{\alpha}^{-O(\log{(1/\bar{\alpha})})}$.

\end{proof}

\section{Distribution-Independent Learning}\label{app:distind}

\subsection{Reduction from Learning Disjoint DNFs}\label{sec:reductions}
In this section we show that distribution-independent learning of coverage functions is at least as hard as distribution-independent learning of disjoint DNF formulas. 

\begin{theorem}[Thm.~\ref{th:dnf-reduction-intro} restated]
Let $\A$ be an algorithm that distribution-independently PAC learns the class of all size-$s$ coverage functions from $\on^n$ to $[0,1]$ in time $T(n, s, \frac{1}{\epsilon})$. Then, there exists an algorithm $\A'$ that PAC learns of $s$-term disjoint DNFs in time $T(2n, s, \frac{2s}{\epsilon})$.
\end{theorem}
\begin{proof}
Let $d = \vee_{i \leq s} T_i$ be a disjoint DNF with $s$ terms. Disjointness of terms implies that $d(x) = \sum_{i \leq s} T_i(x)$ for every $x \in \on^n$. By using de Morgan's law, we have: $d = s - \sum_{i \leq s} D_i$ where each $D_i$ is a disjunction on the negated literals in $T_i$. We will now use a standard reduction \citep{KLV94} through a one-to-one map $m: \on^{n}  \rightarrow \on^{2n}$ and show that there exists a sum of \emph{monotone} disjunctions $d'$ on $\on^{2n}$ such that for every $x \in \on^n$, $d'(m(x)) = s-d(x)$. The mapping $m$ maps $x \in \on^n$ to $y \in \on^{2n}$ such that for each $i \in[n]$, $y_{2i-1} = x_i$ and $y_{2i} = -x_i$. To define $d'$, we modify each disjunction $D_j$ in the representation of $d$ to obtain a monotone disjunction $D_j'$ and set $d' =\sum_{j \leq s} D_j'$. For each $x_i$ that appears in $D_j$ we include $y_{2i-1}$ in $D_j'$ and for each $\neg x_i$ in $D_j$ we include $y_{2i}$. Thus $D_j'$ is a monotone disjunction on $y_1, \ldots, y_{2n}$. It is easy to verify that $d(x) = s - d'(m(x))$ for every $x \in \on^n$. Now $d'/s =  1-d/s$ is a convex combination of monotone disjunctions, that is, a coverage function.

We now describe the reduction itself. As usual, we can assume that the number of terms in the target disjoint DNF, is known to the algorithm. This assumption can be removed via the standard ``guess-and-double" trick. Given random examples drawn from a distribution $\D$ on $\on^n$ and labeled by a disjoint DNF $d$ and $\epsilon > 0$, $\A'$ converts each such example $(x,y)$ to example $(m(x), 1-\frac{y}{s})$. On the modified examples, $\A'$ runs the algorithm $\A$ with error parameter $\epsilon/(2s)$ and obtains a hypothesis $h'$. Finally, $\A$ returns the hypothesis $h(x) = ``s(1-h'(m(x))) \geq 1/2"$ (that is $h(x) = 1$ if $s(1-h'(m(x))) \geq 1/2$ and $h(x)=0$ otherwise).
To establish the correctness of $\A'$ we show that $\pr_{x \sim \D}[d(x) \neq h(x)] \leq \epsilon$. By the definition of $h(x)$ we have that $h(x) \neq d(x)$ only if $|d(x)-s(1-h'(m(x)))|\geq 1/2$. Thus, by the correctness of $\A$, we have
$$
\pr_{x \sim \D}[d(x) \neq h(x)] \leq 2\E_{x \sim \D}[ |(d(x)-(s- sh'(m(x)))|]= 2 s \cdot \E_{x \sim \D}[ |h'(m(x))-(1- \frac{d(x)}{s})|] \leq \epsilon.
$$
Finally, the running time of our simulation is dominated by the running time of $\A$.

\end{proof}

\subsection{Reduction to Learning of Thresholds of Monotone Disjunctions} \label{sec:thresholdreduction}
We give a general reduction of the problem of learning a class of bounded real-valued functions $\C$ to the problem of learning linear thresholds of $\C$. It is likely folklore in the context of PAC learning and was employed by \citet{HRS12} in a reduction from the problem of private data release. Here we give a slightly more involved analysis that also shows that this reduction applies in the agnostic learning setting. The reduction preserves the distribution on the domain and hence can be applied both for learning with respect to a fixed distribution and also in the distribution-independent setting. We remark that this reduction might make the problem (computationally) harder than the original problem. For example while, as we demonstrated, coverage functions are learnable efficiently over the uniform distribution. At the same time linear thresholds of monotone disjunctions appear to be significantly harder to learn, in particular they include monotone CNF formulas that are not known to be learnable efficiently.


\newcommand{\thr}{\mathtt{thr}}

For any $y \in \R$, let $\thr(y): \R \rightarrow \zo$ be the function that is $1$ iff $y \geq 0$. For any class $\C$ of functions mapping $\on^n$ into $[0,1]$, let $\C_\geq$ denote the class of Boolean functions $\{ \thr(c - \theta) \mid c \in \C \text{, } \theta \in [0,1]\}$.
\begin{theorem}
Let $\C$ be a class of functions mapping $\on^n$ into $[0,1]$. Let $\D$ be any fixed distribution on $\on^n$ and suppose $\C_\geq$ is agnostically learnable on $\D$ in time $T(n,\frac{1}{\epsilon})$.
Then, $\C$ is agnostically learnable on $\D$ with $\ell_1$-error $\eps$ in time $O(T(n, \frac{3}{\epsilon})/\epsilon)$.
\label{real2Boolean}
\end{theorem}
\begin{proof}
We will use the algorithm $\A$ that agnostically learns $\C_\geq$ to obtain an algorithm $\A'$ that agnostically learns $\C$.  Let $\P$ be a distribution on $\on^n \rightarrow [0,1]$ whose marginal distribution on $\on^n$ is  $\D$.
Let $c^* \in \C$ be the function that achieves the optimum error, that is $\E_{(x,y) \sim \P}[ |c^*(x) - y|] = \min_{f \in \C}  \E_{(x,y) \sim \P}[ |f(x) - y|]$. For any $\theta \in [0,1]$, let $\P_{\theta}$ denote the distribution on $\on^n \times \zo$ obtained by taking a random sample $(x,y)$ from $\P$ and outputting the sample $(x, \thr(y - \theta))$.

Our algorithm learns $\C$ as follows.
\begin{enumerate}
\item For each $1 \leq i \leq \lfloor \frac{1}{\epsilon} \rfloor = t$ and $\theta=i \cdot \epsilon$ simulate random examples from $\P_{\theta}$ and use $\A$ with an accuracy of $\epsilon$ to learn a hypothesis $h_i$. Notice that the marginal distribution of $\P_{\theta}$ on $\on^n$ is $\D$.
\item Return $h = \epsilon \cdot \sum_{i \in [t]}  h_i$ as the final hypothesis.
\end{enumerate}

To see why $h$ is a good hypothesis for $\P$, first observe that for any $y \in [0,1]$,
\begin{equation}
0 \leq y - \epsilon \cdot \sum_{i \in [t]} \thr(y - i\cdot \epsilon) \leq \epsilon.\label{thresholdsum}\end{equation}
Let $c^*_i :\on^n \rightarrow \zo$ be defined by $c^*_i(x) = \thr(c^*(x) - i \cdot \epsilon)$ for every $i \in [t]$. Thus, $0 \leq c^*(x) - \eps \cdot \sum_{i \in [t]} c^*_i(x) \leq \epsilon$ for every $x \in \on^n$.
Now, since $\A$ returns a hypothesis with error of at most $\epsilon$ higher than the optimum for $\C_\geq$ and $c^*_i \in \C_\geq$ for every $i$ we have:
\begin{align}
 \E_{(x,y) \sim \P} \left[\left|\thr(y-i\cdot \epsilon) - h_i(x)\right|\right] &= \pr_{(x,\ell) \sim \P_{i \cdot \epsilon}}[ \ell \neq h_i(x)] \leq \pr_{(x,\ell) \sim \P_{i \cdot \epsilon}}[ \ell \neq c^*_i(x)] + \epsilon \notag \\
 &= \pr_{(x,y) \sim \P}[\thr(y - i\cdot \epsilon ) \neq c^*_i(x)]  \label{correctnessthreshold1}
\end{align}

for every $i \in [t].$
Now, for any fixed $y$, the number of $i \in [t]$ for which $\thr(y -i \cdot \epsilon) \neq c^*_i(x)$ is at most $\lceil \frac{ |y - c^*(x)|}{\epsilon} \rceil$. Thus,
\begin{equation}
 \sum_{i \in [t]} \pr_{(x,y) \sim \P}[\thr(y - i\cdot \epsilon ) \neq c^*_i(x)] \leq \E_{(x,y) \sim \P}\left[ \left\lceil \frac{ |y - c^*(x)|}{\epsilon} \right\rceil \right] \leq \E_{(x,y) \sim \P}\left[\frac{ |y - c^*(x)|}{\epsilon} \right]  + 1. \label{eq:disc-bound}
\end{equation}
Now, by equations (\ref{thresholdsum}), then (\ref{correctnessthreshold1}), and (\ref{eq:disc-bound}).
\begin{align*}
\E_{(x,y) \sim \P} [ |y - \epsilon  \cdot \sum_{i \in [t]} h_i(x)|] & \leq  \E_{(x,y) \sim \P}\left[\left|y - \epsilon \sum_{i \in [t]} \thr(y-i\cdot \epsilon)\right|\right] + \epsilon \E_{(x,y) \sim \P} \left[\left|\sum_{i \in [t]} \thr(y-i\cdot \epsilon) - \sum_{i \in [t]} h_i(x)\right|\right]
\\ &\leq^{(\ref{thresholdsum})} \eps + \epsilon \sum_{i \in [t]} \E_{(x,y) \sim \P} \left[\left|\thr(y-i\cdot \epsilon) - h_i(x)\right|\right]
\\ &\leq^{ (\ref{correctnessthreshold1})} \eps
 + \eps \sum_{i \in [t]} \left(\pr_{(x,y) \sim \P}\left[\thr(y-i \cdot \epsilon) \neq c^*_i(x)\right] + \eps \right)
\\ &\leq^{(\ref{eq:disc-bound})} \eps
 + \eps \left(\E_{(x,y) \sim \P}\left[\frac{ |y - c^*(x)|}{\epsilon} \right]  + 1 + t \cdot \eps \right)
\\ & \leq  \E_{(x,y) \sim \P} [ |y - c^*(x)|] + 3\epsilon
\end{align*}

This establishes the correctness of our algorithm when used with $\eps/3$ instead of $\eps$. Notice that the running time of the algorithm is dominated by $\lfloor \frac{3}{\epsilon} \rfloor$ runs of $\A$ and thus is at most  $O(\frac{1}{\epsilon} \cdot T(n, \frac{3}{\epsilon}))$. This completes the proof.
\end{proof}
The same reduction clearly also works in the PAC setting (where $\E_{(x,y) \sim \P} [ |y - c^*(x)|] = 0$). We state it below for completeness.

\begin{lemma}
Let $\C$ be a class of functions mapping $\on^n$ into $[0,1]$, let $\D$ be any fixed distribution on $\on^n$ and suppose that $\C_\geq$ is PAC learnable on $\D$ in time $T(n,\frac{1}{\epsilon})$. Then, $\C$ is PAC learnable on $\D$ with $\ell_1$-error $\eps$ in time $O(T(n, \frac{3}{\epsilon})/\epsilon)$.
\end{lemma}

Using these results and Lem.~\ref{disj-rep}, we can now relate the complexity of learning coverage functions with $\ell_1$-error on any fixed distribution to the complexity of PAC learning of the class the class of thresholds of non-negative sums of monotone disjunctions on the same distribution.
\begin{corollary}
Let $\Cv(s)$ denote the set of all coverage functions in $\Cv$ of size at most $s$. Suppose there exists an algorithm that PAC learns the class $\Cv(s)_\geq$ in time $T(n,s, \frac{1}{\epsilon})$ over a distribution $\D$. Then, there exists an algorithm that PAC learns $\Cv(s)$ with $\ell_1$-error in time $O( T(n,s,  \frac{3}{\epsilon})/\epsilon)$ over $\D$.
\end{corollary}

\section{Conclusions}
In this work we described algorithms that provably learn coverage functions efficiently in PAC and PMAC learning models when the distribution is restricted to be uniform. While the uniform distribution assumption is a subject of intensive research in computational learning theory, it is unlikely to ever hold in practical applications. That said, our algorithms make sense and can be used even when the distribution of examples is not uniform (or product) possibly with some tweaks to the parameters. In fact, our algorithms include some of the standard ingredients used in practical machine learning such identification of relevant variables and polynomial regression. Therefore it would be interesting to evaluate the algorithms on real-world data.

Our work also leaves many natural questions about structure and learnability of coverage and related classes of functions (such as submodular, OXS and XOS) open. For example: (1) can coverage functions of unbounded size be PAC/PMAC learned properly and efficiently over the uniform distribution? (2) can OXS functions be PAC/PMAC learned efficiently over the uniform distribution? (3) which other natural distributions can coverage functions be learned on efficiently?


\section*{Acknowledgements}
We are grateful to Jan Vondrak for helpful advice and numerous discussions about this work. We also thank Or Sheffet for bringing to our attention the problem of achieving low average error in differentially private query release for short conjunction queries.
\bibliographystyle{plainnat}
\bibliography{references}
\appendix

\section{Omitted Proofs} \label{app:proofs}
\begin{proof}[Proof of Lemma \ref{disj-rep}]
Suppose $c: \on^n \rightarrow \R^+$ is a coverage function. Then, there exist a universe $U$ and sets $A_1, A_2, \ldots, A_n$ with an associated weight function $w: U \rightarrow \R^+$, such that for any $x \in \on^n$, $c(x) =
 \sum_{u \in \cup_{j:x_j = -1} A_j} w(u)$. For any $u \in U$, let $S_u =\{ j \in [n] \cond u \in A_j\}$. Then, $c(x) = \sum_{u \in U} w(u) \cdot \ORR_{S_u}(x)$ since $\ORR_{S_u}(x) = 1$ if and only if $x_j = -1$ for some $j \in S$. This yields a representation of $c$ as a linear combination of $|U|$ disjunctions.

For the converse, now let $c$ be any function such that there exist non-negative coefficients $\alpha_S$ for every $S \subseteq [n], S \neq \emptyset$ such that $c(x) = \sum_{S \subseteq [n], S \neq \emptyset} \alpha_S \cdot \ORR_S(x)$. We will now construct a universe $U$ and sets $A_1, A_2, \ldots ,A_n$ with an associated weight function $w: U \rightarrow [0,1]$ and show that $f(x) = \sum_{u \in \cup_{i: x_i = -1} A_i} w(u)$. For every non zero $\alpha_S$ add a new element $u_S$ to each $A_i$ such that $i \in S$ and let $w(u_S) = \alpha_S$. Let $U = \cup_{1 \leq i \leq n} A_i$. Now, by our construction, $c(x) = \sum_{u_S: \alpha_S \neq 0} \alpha_S \ORR_S(x) = \sum_{u \in \cup_{i:x_i = -1} A_i} w(u)$ as required.
\end{proof}

\begin{proof}[Proof of Lem.~\ref{spectralconc}]
First, observe that:
\begin{align*}
\sum_{T \notin \T} \hat{f}(T)^2 &\leq \max_{T \notin \T} |\hat{f}(T)| \cdot \sum_{T \in \T} |\hat{f}(T)|\\
&\leq \frac{\epsilon}{2L} \sum_{T \subseteq [n]} |\hat{f}(T)| = \frac{\epsilon}{2L} \|\hat{f}\|_1 \leq \epsilon/2.
\end{align*}
 Thus $f$ is $\epsilon/2$-concentrated on $\T$. Further, $|\T|\leq \frac{2\|f\|_1}{\epsilon}$ since for each $ T\in \T$,  $|\hat{f}(T)| \geq \frac{\epsilon}{2\|f\|_1}$.

By Plancherel's theorem, we have:
$$\E[(f(x) - \sum_{S \in \S} \tilde{f}(S) \cdot \chi_S(x))^2] = \sum_{S \in \S} [ (\hat{f}(S) - \tilde{f}(S))^2] + \sum_{S \notin \S} \hat{f}(S)^2\ .$$
The first term,
\begin{equation} \sum_{S \in \S} [ (\hat{f}(S) - \tilde{f}(S) )^2] \leq \max_{S \in \S} \{ |\hat{f}(S) - \tilde{f}(S)| \} \cdot \sum_{S \in \S} \{ |\hat{f}(S) - \tilde{f}(S)| \} \leq \frac{\epsilon}{6L} \cdot \sum_{S \in \S} |\hat{f}(S) - \tilde{f}(S)| . \label{approxspecguar}\end{equation}

For each $S \in \T$, $|\tilde{f}(S)| \geq \frac{\epsilon}{3L}$ and $|\hat{f}(S) -  \tilde{f}(S)|\leq \frac{\epsilon}{6L}$. Thus
$|\hat{f}(S)| \geq \frac{\epsilon}{6L} \geq |\hat{f}(S) - \tilde{f}(S)|$ and $\sum_{S \in \S} |\hat{f}(S) - \tilde{f}(S)| \leq \|\hat{f}\|_1 = L$. Using equation \eqref{approxspecguar}, this gives $\sum_{S \in \S} [(\hat{f}(S) - \tilde{f}(S) )^2] \leq \epsilon/6 $.

For the second term, $\sum_{S \notin \S} \hat{f}(S)^2 \leq \epsilon/2$ as $\S \supset \T$ and $f$ is $\epsilon/2$-concentrated on $\T$. Thus, $\E[(f(x) - \sum_{S \in \S} \tilde{f}(S) \cdot \chi_S(x))^2] \leq \epsilon$.

Finally, by Jensen's inequality: $$\E[|f(x) - \sum_{S \in \S} \tilde{f}(S) \cdot \chi_S(x)|] \leq \sqrt{ \E[(f(x) - \sum_{S \in \S} \tilde{f}(S) \cdot \chi_S(x))^2]} \leq \sqrt{\eps}.$$
\end{proof}

\begin{proof}[Proof of Lem.~\ref{lem:disj2cov}]
Using Lemma \ref{disj-rep}, let $c = \sum_{S \subseteq [n]} \alpha_S \cdot \ORR_S$ for some non-negative $\alpha_S$, $S \subseteq [n]$ such that $\sum_{S \subseteq [n]} \alpha_S \leq 1$.

Set $\beta_j = \sum_{S \subseteq [n]} \alpha_S \cdot \beta_j^S $ for every $1 \leq j \leq m$. We verify that $f = \sum_{j = 1}^m \beta_j \cdot \phi_j$ has the required approximation guarantee.

\begin{align*}\E_{x \sim \D}[ |c(x) - f(x)|] &= \E_{x \sim \D}\left[ \left|\sum_{S \subseteq [n]} \alpha_S \cdot (\ORR_S - \sum_{j = 1}^m \beta_j \cdot \phi_j(x)) \right|\right] \\
&\leq  \sum_{S \subseteq [n]} \alpha_S \cdot  \E_{x \sim \D} \left[ \left|\ORR_S(x) - \sum_{j = 1}^m \beta_j^S \cdot \phi_j(x)\right|\right]\\
&\leq \sum_{S \subseteq [n]} \alpha_S \cdot \epsilon \leq \epsilon.\end{align*}
\end{proof}

\end{document}